\documentclass[10pt,twocolumn,letterpaper]{article}

\usepackage{cvpr}              
\usepackage[accsupp]{axessibility}

\usepackage[dvipsnames]{xcolor}

\definecolor{cvprblue}{rgb}{0.21,0.49,0.74}
\usepackage[pagebackref,breaklinks,color links,citecolor=cvprblue]{hyperref}
\usepackage{makecell}
\usepackage{multirow, multicol}
\usepackage{adjustbox}
\usepackage{boldline}
\usepackage{pifont}

\usepackage{array}
\usepackage{multirow}
\usepackage{makecell}
\usepackage{tabularray}
\usepackage{algorithmic}
\usepackage[accsupp]{axessibility}
\usepackage{algorithm}
\usepackage{colortbl}  
\usepackage{xcolor}  
\usepackage{amsthm}

\def\paperID{1021} 
\def\confName{CVPR}
\def\confYear{2025}

\title{Test-Time Domain Generalization via Universe Learning: A Multi-Graph Matching Approach for Medical Image Segmentation}

\author{Xingguo Lv\textsuperscript{1,2}~ Xingbo Dong\textsuperscript{1}\footnotemark[2] ~ Liwen Wang\textsuperscript{1} ~ Jiewen Yang\textsuperscript{3} ~ Lei Zhao\textsuperscript{2}\\  Bin Pu\textsuperscript{3}\footnotemark[2] ~ Zhe Jin\textsuperscript{1} ~ Xuejun Li\textsuperscript{1}\\
\textsuperscript{1} Anhui Provincial International Joint Research Center for Advanced Technology in Medical Imaging,\\ Anhui University.
\textsuperscript{2} Hunan University.
\textsuperscript{3} The Hong Kong University of Science and Technology.\\
{\tt\small lvxg@stu.ahu.edu.cn, xingbo.dong@ahu.edu.cn, eebinpu@ust.hk}
}

\begin{document}
\maketitle
\begin{abstract}
Despite domain generalization (DG) has significantly addressed the performance degradation of pre-trained models caused by domain shifts, it often falls short in real-world deployment. Test-time adaptation (TTA), which adjusts a learned model using unlabeled test data, presents a promising solution. However, most existing TTA methods struggle to deliver strong performance in medical image segmentation, primarily because they overlook the crucial prior knowledge inherent to medical images. To address this challenge, we incorporate morphological information and propose a framework based on multi-graph matching. 
Specifically, we introduce learnable universe embeddings that integrate morphological priors during multi-source training, along with novel unsupervised test-time paradigms for domain adaptation. This approach guarantees cycle-consistency in multi-matching while enabling the model to more effectively capture the invariant priors of unseen data, significantly mitigating the effects of domain shifts.
Extensive experiments demonstrate that our method outperforms other state-of-the-art approaches on two medical image segmentation benchmarks for both multi-source and single-source domain generalization tasks. 
The source code is available at 
\href{https://github.com/Yore0/TTDG-MGM}
{https://github.com/Yore0/TTDG-MGM}.
\end{abstract}

\renewcommand{\thefootnote}{\fnsymbol{footnote}}
\footnotetext[2]{Corresponding authors.}
\section{Introduction}
\label{sec:intro}

\begin{figure}[!t]
    \centering
    \includegraphics[width=0.98\linewidth]{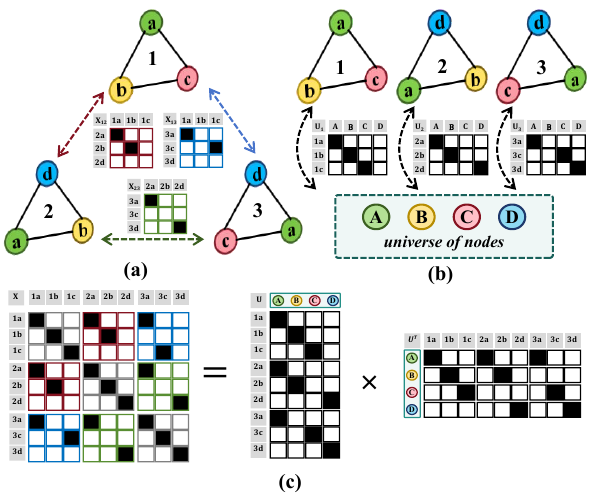}
    \caption{Illustration of multi-graph matching using the \textit{universe of nodes}.
 (a) Direct pairwise matching between graphs 1, 2, and 3. (b) Multi-graph matching with the \textit{universe of nodes}. (c) Equivalence between (a) and (b) (see \textit{Lemma} 1 in Sec.~\ref{lemma1} for details).}
    \vspace{-10pt}
    \label{fig:moti}
\end{figure}

In the field of medical image processing, semantic segmentation is a technique that enables the precise identification and quantitative analysis of target regions. 
Currently, a considerable number of medical models can effectively perform segmentation tasks~\cite{azad2024medical,ma2024segment}. Nevertheless, when pretrained on source datasets, their performance often declines in real-world deployment due to domain shifts~\cite{moreno2012unifying,zhang2021adaptive} caused by differences in imaging devices, protocols, patient demographics, and preprocessing methods, which impact the model's generalization ability.
To address domain shifts, previous researches in domain adaptation, such as domain generalization (DG)~\cite{wang2022generalizing,zhou2022domain,yoon2023domain}, have focused on designing sophisticated models that train jointly on source and target domains or across multiple styled domains.
Despite the progress, these methods fall short in practice, as no source domains can fully capture all real-world variations, and retraining for each new data is impractical. Consequently, a more practical approach is to adaptively fine-tune pre-train model using the information from the unlabeled unseen data during the test, which is referred to test-time adaptation (TTA)~\cite{liang2024comprehensive}.


Medical images, due to their reliance on specific physical principles (such as X-ray absorption~\cite{de2001high}, magnetic resonance~\cite{slichter2013principles}, sound wave reflection~\cite{ter2007therapeutic}, etc.), are fundamentally different from images capturing visible light in natural scenes. Moreover, the physiological characteristics of specific anatomical regions exhibit relatively stable shapes and spatial layouts, and the morphology of these structures is predictable~\cite{pu2024m3,pmlr-v235-pu24b}. These prior insights provide valuable information that is particularly useful in medical imaging tasks. 
In areas such as segmentation~\cite{yang2023graphecho,liu2022single}, registration~\cite{kim2022diffusemorph}, and detection~\cite{pu2024m3,pmlr-v235-pu24b}, many well-established approaches have been proposed that combine above handcrafted designs with neural network-based methods.

We combine the aforementioned TTA strategies with the prior knowledge of medical images and design a novel Test-Time Domain Generalization (TTDG) framework utilizing graph structures.
The construction of graph enables a more robust representation of the morphological priors inherent to medical images. In complex medical scenarios, the aggregation mechanism of nodes and edges in graph ensures the effective preservation of crucial contextual information. Previous research has utilised pairwise graph matching~\cite{li2023sigma++,sarlin2020superglue,pu2024m3} to address domain shifts in medical imaging. However, these methods are limited to aligning features between two domains, which poses challenges in scenarios involving multiple domains commonly encountered in real-world settings. Multi-graph matching overcomes these limitations by establishing connections between multiple graphs and integrating information from diverse domains, thereby enabling the capture of more intricate patterns and structural relationships while maintaining cross-domain consistency. This process facilitates the acquisition of global invariant features from medical images.

In multi-graph matching, the \textit{universe of nodes}~\cite{pachauri2013solving,tron2017fast} represents the full-set of the nodes from all graphs (refer to Fig.~\ref{fig:moti}). The aforementioned morphological priors are embedded into this universe, derived from a multitude of hospitals, devices, and modalities. It offers three main advantages: (1) \textbf{Cross-domain consistency.} By mapping the nodes from different graphs into the universe, we ensure that all graphs are aligned according to the same reference standard, i.e. the priors. (2) \textbf{Cycle-consistency in multi-graph matching.} This guarantees global coherence across multiple graphs, avoiding conflicts in local node alignments. (3) \textbf{Handling partial matching and missing nodes.} Even when certain nodes are absent in some graphs, matching them to virtual nodes in the universe addresses the issue, enhancing robustness to incomplete graph structures. 

Our contributions can be summarized as follows:

\begin{itemize}
\item We propose the first (to our best knowledge) multi-graph matching framework for test-time domain generalization in medical image segmentation tasks, effectively mitigating performance degradation caused by domain shifts during the testing phase.

\item  By designing learnable universe embeddings, we integrate the morphological priors of medical images into the graph matching process during source training, while ensuring the cycle-consistency constraint. This approach enables joint optimization and promotes the learning of domain-invariant features.

\item We design a novel, well-initialized unsupervised testing adaptation paradigm that integrates prior knowledge and allows seamless deployment during adaptation, effectively addressing domain shifts. 

\item  We conduct extensive experiments on two typical medical datasets, demonstrating that our method performs competitively against state-of-the-art approaches for both multi-source and single-source domain generalization.

\end{itemize}

\section{Preliminary \& Related Work}

The following section will introduce the preliminaries of multi-graph matching, while also reviewing the most relevant existing literature on the subject. Subsequently, domain generalization and test-time adaptation will be discussed.
\subsection{Multi-Graph Matching}
We consider $m \in \mathbb{N}$ different graphs $\mathcal{G}_1, \mathcal{G}_2, \cdots, \mathcal{G}_m$, $i \in [m]$, $[m] := \{1,\cdots,m\}$. For each $\mathcal{G}_i=(\mathcal{V}_i, \mathcal{A}_i)$, $\mathcal{V}_i\in \mathbb{R}^{n_i \times h}$ represents the $h$-dimensional feature of $n_i$ nodes, and adjacency matrix $\mathcal{A}_i\in \mathbb{R}^{n_i \times n_i}$ encodes the connectivity between nodes, represented as the set of edges. For any two graphs $\mathcal{G}_i$ and $\mathcal{G}_j$, the set of $n_i \times n_j$ partial permutation matrices $\mathbb{P}_{n_i n_j}$ is defined as 
\begin{equation}
    \mathbb{P}_{n_i n_j} := \{\bold{X}\in \{0,1\}^{n_i\times n_j}: \bold{X}\bold{1}_{n_j} \leq \bold{1}_{n_i}, \bold{X}^{\mathsf{T}}\bold{1}_{n_i} \leq \bold{1}_{n_j}\},
\end{equation}
where $\bold{1}_{n_i}$ denotes a $n_i$-dimensional column vector whose elements are all ones. The assignment matrix $\bold{X}_{ij} \in \mathbb{P}_{n_i n_j}$ between a pair of graphs $(\mathcal{G}_i, \mathcal{G}_j )$ denotes a meaningful correspondence that encodes the matching. 

When considering the synchronisation matching~\cite{pachauri2013solving,maset2017practical} of multiple graphs $(m>2)$, relying on local matches between graph pairs can readily result in erroneous correspondences that contradict each other globally~\cite{yan2013joint}. Consequently, the multi-graph matching has been addressed in terms of simultaneously solving under the constraints of cycle-consistency~\cite{yan2015multi,yan2014graduated}.

\noindent \textit{\textbf{Definition 1 Cycle-consistency.} The matching among $\mathcal{G}_1, \mathcal{G}_2, \cdots, \mathcal{G}_m$ is cycle-consistent (partial transitivity), if }
\begin{equation}
\label{cycle_consist}
    \bold{X}_{ik}\bold{X}_{kj} \leq \bold{X}_{ij}, \qquad \forall{i,j,k\in [m]}.
\end{equation}

In contrast to full matching, i.e. in Eq. (\ref{cycle_consist}) the inequalities become equalities and $n_i = n_j$, partial matching requires only that the pairwise matching combinations in the cycle form a subset of the identity matching~\cite{bernard2019synchronisation}. 
An iterative refinement strategy~\cite{nguyen2011optimization} was employed to enhance pairwise matchings by ensuring global mapping consistency. In~\cite{pachauri2013solving}, the authors proposed a method to achieve cycle-consistency based on spectral decomposition of the assignment matrix. In addition, methods such as convex programming~\cite{chen2014near}, low-rank matrix recovery~\cite{zhou2015multi}, and spectral decomposition~\cite{maset2017practical} have been proposed to address the synchronization of partial matchings.

Instead of explicitly modeling the cubic number of non-convex quadratic constraints, a more efficient approach to enforcing cycle consistency is the use of the \textit{universe of nodes}~\cite{tron2017fast}. In a universe comprising nodes of size $d \in \mathbb{N}$, the matching between graphs can be decomposed by matching each graph $\mathcal{G}_i$ to the space of universe (refer to Fig.~\ref{fig:moti}). We denote the \textit{universe matchings} as follows:
\begin{equation}
    \mathbb{U}_{n_id} := \{ \bold{U} \in \{0,1\}^{n_i\times d}: \bold{U}\bold{1}_d=\bold{1}_{n_i}, \bold{U}^{\mathsf{T}}\bold{1}_{n_i} \leq \bold{1}_d \}.
\end{equation}
 \noindent \textit{\textbf{Lemma 1 Cycle-consistency, universe matching.} \label{lemma1} The pairwise (partial) matching matrices $\{ \bold{X}_{ij}\}_{i,j=1}^{m} $ is cycle-consistent iff there exists a collection of universe matching $\{ \bold{U}_i \in \mathbb{U}_{n_id} \}_{i=1}^{m}$, such that for each $\bold{X}_{ij}$, it holds that $\bold{X}_{ij} = \bold{U}_i \bold{U}_{j}^{\mathsf{T}}$. }

The proof will be presented in the appendix. By projecting each node of a graph to the universe of nodes and identifying the universe matchings $\mathbb{U}_{n_id}$, it is ensured that the cycle-consistency constraint will be satisfied throughout the multi-graph matching process. In~\cite{wang2020graduated,wang2023unsupervised}, the authors proposed a gradual assignment procedure to achieve soft matching and clustering through the universe matching. In~\cite{bernard2019hippi}, the universe of nodes is leveraged to incorporate geometric consistency, ensuring both point scaling and convergence. \cite{nurlanov2023universe} utilized the universe matching to facilitate effective partial multi-graph matching.

\noindent \textit{\textbf{Definition 2 Multi-Matching Koopmans-Beckmann's Quadratic Assignment Problem (KB-QAP)~\cite{koopmans1957assignment}.} Multi-graph matching is formulated with KB-QAP, by summing KB-QAP objectives among all pairs of graphs:
\begin{align}
\label{eq:mgqap}
    \underset {\bold{X}_{ij}} {\max} \underset {i,j\in [m]}\sum ( & \lambda \cdot \text{tr}(\bold{X}_{ij}^{\mathsf{T}}\mathcal{A}_i\bold{X}_{ij}\mathcal{A}_j) + \text{tr}(\bold{X}_{ij}^{\mathsf{T}} \bold{M}_{ij})), \\ \nonumber
    s.t.  \quad & \bold{X}_{ij}\in \{ 0,1\}^{n_i\times n_j}, \bold{X}_{ij} \in \mathbb{P}_{n_i n_j},  \\ \nonumber
    & \bold{X}_{ik} \bold{X}_{kj} \leq \bold{X}_{ij}, \quad \forall i,j,k\in [m].
\end{align} }

\noindent In Eq. (\ref{eq:mgqap}), $\lambda$ is a scaling factor for edge-to-edge similarity, and $\bold{M}_{ij}$ is the node-to-node similarity between $\mathcal{G}_i, \mathcal{G}_j$.

\subsection{Domain Generalization}

Domain generalization addresses a challenging scenario where one or more different but related domains are provided, with the objective of training a model that can generalize effectively to an unseen test domain~\cite{wang2022generalizing,zhou2022domain,yoon2023domain}. Medical image analysis encounters significant challenges due to factors such as variability in image appearance, the complexity and high dimensionality of the data, difficulties in data acquisition, and issues related to data organization, labeling, safety, and privacy. Previous research has proposed DG algorithms across multiple levels, including data-level~\cite{yu2023san,zhang2020generalizing}, feature-level~\cite{bi2023mi,li2018domain}, model-level~\cite{gu2023cddsa,peng2019moment,zuo2021attention}, and analysis-level~\cite{liu2022single,qiao2020learning}. Yu \textit{et al.}~\cite{yu2023san} proposed a U-Net z-score nomalization network for the stroke lesion segmentation. Liu \textit{et al.}~\cite{liu2022single} employed dictionary learning for prostate and fundus segmentation by creating a shape dictionary composed of template masks. In~\cite{bi2023mi}, the authors utilized mutual information to differentiate between domain-invariant features and domain-specific ones in ultrasound image segmentation. Gu \textit{et al.}~\cite{gu2023cddsa} proposed a domain-style contrastive learning approach that disentangles an image into invariant representations and style codes for DG. 

\subsection{Test-Time Adaptation}

Test-time adaptation is an emerging paradigm that allows a pre-trained model to adapt to unlabeled data during the testing phase, before making predictions~\cite{liang2024comprehensive}.
Several TTA methods have been proposed recently, utilizing techniques such as self-supervised learning~\cite{chen2023improved,kundu2022concurrent,zhang2022memo}, batch normalization calibration~\cite{zhang2023domainadaptor,chen2024each,zou2022learning,schneider2020improving,nado2020evaluating}, and input data adaptation~\cite{shu2022test,karani2021test} to achieve better test-time performance. Typically, VPTTA~\cite{chen2024each} is a method that freezes the pre-trained model and generates low-frequency prompts for each image during inference in medical image segmentation. In~\cite{karani2021test}, the authors proposed designing two CNN-based sub-networks along with an image normalization network. During test-time training, the image normalization network was adapted for each image. \cite{zhang2023domainadaptor} proposed a dynamic mixture coefficient and a statistical transformation operation to adaptively merge the training and testing statistics of the normalization layers. Additionally, the authors design an entropy minimization loss to address the issue of domain shifts. 

\section{Methodology}
\begin{figure*}[!t]
    \centering
    \includegraphics[width=0.98\linewidth]{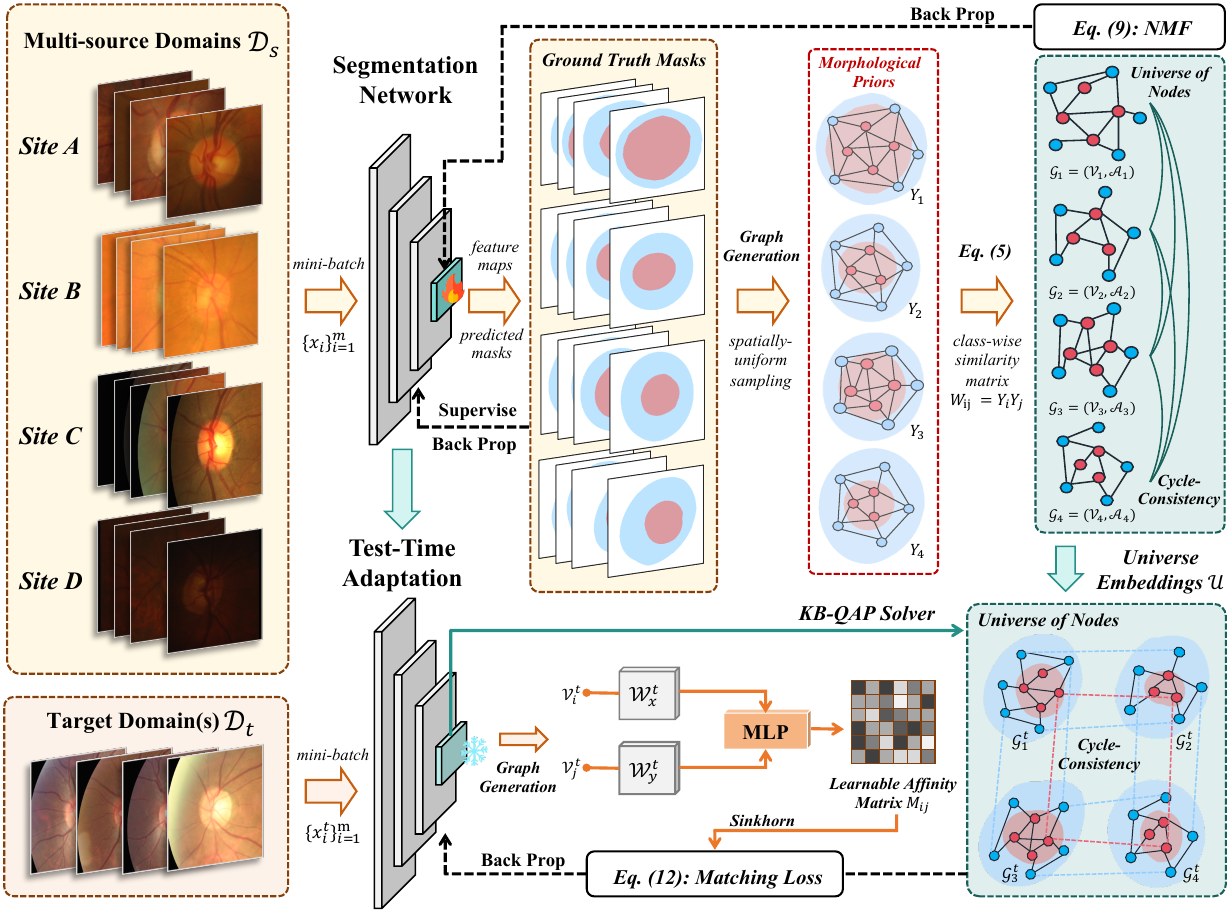}
    \caption{\textbf{Overview of our TTDG framework.} During source model training, data from different domains are jointly used to train the Segmentation Network (feature extractor and segmentation head). Feature maps and ground truth masks are utilized to construct graphs $\mathcal{G}_i$ ($i=4$ in this figure legend) and corresponding labels $Y_i$ (Sec.~\ref{sec:graph_generation}), with universe embeddings learned via back-propagation, incorporating morphological priors (Sec.~\ref{sec:training}).
    At test time, multi-graph matching is performed on all target domain images in each batch. Despite style differences, these images share common morphological patterns. Universe embeddings are frozen as prior knowledge to guide the matching, and the segmentation network is fine-tuned via back-propagation for efficient adaptation (Sec.~\ref{sec:testing}).}
     \vspace{-0.5cm}
    \label{fig:pipeline}
\end{figure*}
This section provides an overview of how multi-graph matching methods can facilitate training for TTDG tasks. Given inputs consisting of images from $S$ source domains $(S \ge 1)$, denoted as $\mathcal{D}_s = \{ D_1, D_2, \cdots, D_S \}$, the objective is to enable the model to make more accurate predictions on $T$ unseen target domains $ (T\ge 1)$, denoted as $\mathcal{D}_t = \{ D_1, D_2, \cdots, D_T \}$, during the testing phase.

Inspired by~\cite{pmlr-v235-pu24b}, we aim to effectively leverage morphological prior through graph construction. However, unlike the UDA tasks, which typically involve only two domains, our training process simultaneously handles multiple labeled domains. Pairwise matching across multiple domains can lead to the challenges discussed in Sec.~\ref{sec:intro}. To tackle this, we propose learning domain-invariant latent representations from a multi-graph matching perspective, thereby facilitating rapid adaptation to unseen domains during the TTA phase.

\subsection{Construction of Graph}
\label{sec:graph_generation}

Rather than using conventional graph-matching algorithms that depend on explicitly defined keypoints as graph nodes~\cite{gao2021deep,wang2019learning,wang2021neural}, accurately annotating keypoints in medical imaging is both more costly and less practical than in natural images.
Therefore, given a mini-batch $\{x_i\}_{i=1}^{m} \in \mathcal{D}_s$, as shown in Fig.~\ref{fig:pipeline}, it is first passed through a general feature extractor, such as ResNet~\cite{he2016deep}, to obtain visual features.
We then perform spatially-uniform sampling~\cite{li2022sigma} of pixels within the ground-truth masks at each feature level to obtain $n_i$ foreground nodes with its corresponding labels ${Y}_i\in \mathbb{Z}^{n_i}$.
After extracting these fine-grained visual features, we employ a nonlinear projection to transform the visual space into the $h$-dimensional graph space. This approach generates the feature of nodes $\{\mathcal{V}_i\in \mathbb{R}^{n_i \times h}\}_{i=1}^m$ that more effectively preserve the semantic characteristics. 

For the weighted adjacency matrix encoding structural information, we use Dropedge~\cite{rongdropedge} to reduce potential bias caused by the frequent visual correspondences, preventing the model from over-relying on specific matches. The formulation is as follows:
\begin{equation}
    \mathcal{A}_i = Dropedge\{ softmax[\mathcal{V}_i \mathcal{W}_{x}\cdot(\mathcal{V}_i \mathcal{W}_{y})^{\mathsf{T}}]\odot D^{-1}\},
\end{equation}
where $\mathcal{W}_{x}$ and $\mathcal{W}_{y}$ are two learnable linear projection, $D=\frac{\mathcal{V}_i \mathcal{V}_i^{\mathsf{T}}}{||\mathcal{V}_i||_2}$ is the cosine distance matrix. Edges with larger weights indicate nodes that are closer to each other, so we use the inverse of $D$ and apply the hadamard product $\odot$ to combine node similarity with geometric proximity. The diagonal elements of $\mathcal{A}_i$ are set to zero. Based on the above steps, we obtain a corresponding graph $ \mathcal{G}_i=(\mathcal{V}_i, \mathcal{A}_i)$ for each input $x_i$.

\subsection{Formulation of the Universe Embeddings}
\label{sec:training}
In Sec.~\ref{lemma1}, it was demonstrated that by matching each node of a graph to the \textit{universe of nodes}, we can maintain cycle-consistency while avoiding the computational burden associated with cubic non-convex constraints. While a randomly initialized matrix can learn accurate universe matchings $\mathbb{U}_{n_id}$ in a supervised setting, we introduce universe embeddings $\mathcal{U} \in \mathbb{R}^{d\times h}$, which serve as learnable latent representations to enhance adaptability in unsupervised multi-domain scenarios.

$\mathcal{U}$ is initialized as $1/d + 10^{-3}z$, where $z \sim N(0, 1)$, following the setting in~\cite{wang2020graduated}. For each $\mathcal{G}_i$, we learn a unique matching between nodes in $\mathcal{V}_i$ and the corresponding in the \textit{universe of nodes}, as follows: $U_i = Sinkhorn(\mathcal{V}_i \ \mathcal{U}^{\mathsf{T}}, \tau) \in \mathbb{U}_{n_i d}$. $Sinkhorn(X, \tau)$ refers to the Sinkhorn algorithm~\cite{sinkhorn1964relationship,cuturi2013sinkhorn}, which applies a relaxed projection with entropic regularization~\cite{gold1996softmax} to normalize the matrix $X$, resulting in a doubly-stochastic matrix, and $\tau \in (0,+\infty)$ is the regularization factor. For the obtained $U_i$, the $d\times d$ matrix $U_i^{\mathsf{T}} \mathcal{A}_i U_i$ is a row and column reordering of $\mathcal{A}_i$ based on the \textit{universe of nodes}. Accordingly, the agreement between two adjacency matrices $\mathcal{A}_i$ and $\mathcal{A}_j$, reordered according to their respective universe matching assignment matrices $U_i$ and $U_j$, can be quantified using the Frobenius inner product namely $\langle U_i^{\mathsf{T}} \mathcal{A}_i U_i, U_j^{\mathsf{T}} \mathcal{A}_j U_j\rangle$. Let $\bold{U} = [U_1^{\mathsf{T}}, \cdots, U_m^{\mathsf{T}}]^{\mathsf{T}} \in \mathbb{R}^{n\times d}$, where $n = \sum_{i=1}^{m}n_i$. The block-diagonal multi-adjacency matrix is defined as $\bold{A} = diag(\mathcal{A}_1, \cdots, \mathcal{A}_m) \in \mathbb{R}^{n\times n}$. We can get the multi-matching formulation as:
\vspace{-5pt}
\begin{align}
    f(\bold{U}) := &\ \underset{U_1, \dots, U_m}{\max} \quad \sum_{i,j \in [m]} \langle U_i^{\mathsf{T}} \mathcal{A}_i U_i, U_j^{\mathsf{T}} \mathcal{A}_j U_j \rangle \nonumber \\ 
     = &\quad \underset{\bold{U} \in \mathbb{U}}{\max} \quad \text{tr}(\bold{U}^{\mathsf{T}} \bold{A} \bold{U} \bold{U}^{\mathsf{T}} \bold{A} \bold{U}) \label{eq6}.
\end{align}

It was observed that in the original multi-graph matching optimization problem, graphs are typically constructed for each individual instance, with all the nodes within a graph belonging to the same category. However, in medical imaging, constructing separate graphs for each organ would not only greatly increase the computational complexity of graph matching but also result in misalignments between graphs from different categories, significantly compromising matching accuracy. To address this limitation, we introduce a \textit{class-wise similarity matrix} $\bold{W}$ to incorporate class-aware label information for each node. This ensures that the nodes are correctly aligned according to their respective categories during multi-graph matching.
Specifically, we define $W_{ij} = {Y}_i {Y}_j^{\mathsf{T}} $ and $\bold{W} = [W_{ij}]_{ij} \in \mathbb{Z}^{(\sum_{i=1}^{m}n_i)\times(\sum_{j=1}^{m}n_j)}$. The class-aware similarity matrix $\bold{\tilde{A}} = \bold{W}^{\mathsf{T}} \bold{A}\bold{W}$. Consequently, Eq. (\ref{eq6}) can be transformed into the following form
\vspace{-5pt}
\begin{equation}
    \underset{\bold{U} \in \mathbb{U}}{\max} \quad \text{tr}(\bold{U}^{\mathsf{T}} \bold{\tilde{A}} \bold{U} \bold{U}^{\mathsf{T}} \bold{\tilde{A}} \bold{U}).
\end{equation}


The higher-order projected power iteration (HiPPI)~\cite{bernard2019hippi} is employed to iteratively solve $\bold{U}$ to obtain a stable convergence form:
\vspace{-5pt}
\begin{equation}
    U_i^\ast = U^t_i \quad s.t. \quad || U^t_i - U^{t-1}_i|| < \theta, \quad \forall i \in [m].
\end{equation}

Empirically, we set the convergence threshold $\theta=10^{-5}$. In order to guarantee the optimal learning of universe embeddings $\mathcal{U}$, we incorporate the \textit{non-negative matrix factorisation} (NMF)~\cite{lee1999learning,bernard2019synchronisation} and define the overall optimization objective as: 
\vspace{-5pt}
\begin{equation}
    L(\mathcal{U}) = \sum_{i\in [m]} || {U}^\ast_i - \mathcal{V}_i \ \mathcal{U}^{\mathsf{T}} ||^2_F + \alpha ||\mathcal{U} ||_F^2,
\end{equation}
where $||\cdot||^2_F$ represents the Frobenius norm, and $\alpha$ is a tuning parameter. Since only matrix multiplication and element-wise division are involved, the Sinkhorn operation is fully differentiable. This allows for effective updates to $\mathcal{U}$ through back-propagation, enabling the incorporation of morphological prior knowledge.

\subsection{Testing Paradigm on the Target Domains}
\label{sec:testing}
During the Test-Time Adaptation (TTA) phase, we only have access to unlabeled data from the unseen target domains $\mathcal{D}_t$. Due to distribution shifts, the performance of models trained on $\mathcal{D}_s$ can degrade significantly. To address this issue, we employ the multi-graph matching on each mini-batch $\{x_i^t\}_{i=1}^{m} \in \mathcal{D}_t$. By utilizing universe embeddings $\mathcal{U}$ learned from $\mathcal{D}_s$ as described in Sec.~\ref{sec:training}, it is guaranteed cycle-consistency, enabling the network to focus on domain-invariant features guided by morphological priors from medical images. 
In this phase, $\mathcal{U}$ is frozen to prevent error accumulation from affecting the prior information. The feature extractor and segmentation head (segmentation network in Fig.~\ref{fig:pipeline}) are optimized through gradient back-propagation, allowing the model to adapt effectively to the new domains.

\noindent \textbf{Unsupervised Multi-Matching.} Following the same graph construction described in Sec.~\ref{sec:graph_generation}, we obtain the $\mathcal{G}_i^t = (\mathcal{V}_i^t, \mathcal{A}_i^t)$ for each $x_i^t$, where $\mathcal{V}_i^t \in \mathbb{R}^{n_i\times h}$ represents the $n_i$ class-aware nodes and $\mathcal{A}_i^t \in \mathbb{R}^{n_i\times n_i}$ is the adjacency matrix encoding the structural information. 
We first introduce a network-learned affinity matrix $M_{ij} = f_{mlp}\{\mathcal{V}_i^t \mathcal{W}_{x}^t\cdot(\mathcal{V}_j^t \mathcal{W}_{y}^t)^{\mathsf{T}}\} \in \mathbb{R}^{n_i\times n_j}$ in graph pair $\mathcal{G}_i$ and $\mathcal{G}_j$, which computes the similarity between their corresponding node features. $\mathcal{W}_{x}^t$ and $\mathcal{W}_{y}^t$ are two learnable linear projection, $f_{mlp}$ is a multi-layer perception (MLP). Subsequently, the matching between $\mathcal{G}_i$ and the universe of size $d$ is obtained through the inner-product between $\mathcal{V}_i^t$ and the pre-trained universe embeddings $\mathcal{U}$, in accordance with the same procedure $U_i = Sinkhorn(\mathcal{V}_i^t \ \mathcal{U}^{\mathsf{T}}, \tau) \in \mathbb{U}_{n_i d}$ in Sec.~\ref{sec:training}.

In~\cite{wang2020graduated,wang2023unsupervised}, the authors propose solving the multi-matching KB-QAP problem presented in Eq. (\ref{eq:mgqap}) (denoted as $J$) using a Taylor expansion. Based on the assumptions in \textit{Lemma 1} of Sec.~\ref{lemma1}, we can reformulate $J$ in Taylor series as follows:
\vspace{-5pt}
\begin{equation}
    J(U) \approx J(U^0) + \sum_i \frac{\partial J}{\partial U_i} \bigg|_{U=U^0} (U_i-U_i^0) + \cdots.
\end{equation}

$J$ can be approximated as $m$ linear optimization problems through the first-order Taylor expansion, which can be efficiently solved using the Sinkhorn. The gradient of $J$ with respect to $U_i$, denoted as
\vspace{-5pt}
\begin{equation}
\label{eq:v}
    V_i = \frac{\partial J}{\partial U_i} = \sum_{j\in [m]} (\lambda \mathcal{A}_i U_i U_j^{\mathsf{T}} \mathcal{A}_j U_j + M_{ij} U_j),
\end{equation}
$V_i$ ensures that $U_i$ gradually converges to a high-quality discrete solution.

Finally, the converged $\{U_i\}$ fine-tune the segmentation network with the matching loss $L_{mat}$: 
\vspace{-5pt}
\begin{align}
    L_{mat} = & \sum_{i,j \in [m]} \left[ -\hat{M}_{ij}^\gamma (1 - U_i U_j^{\mathsf{T}}) \log(U_i U_j^{\mathsf{T}}) \right] \nonumber \\
     - & \sum_{i,j \in [m]} \left[ (1 - \hat{M}_{ij})^\gamma U_i U_j^{\mathsf{T}} \log(1 - U_i U_j^{\mathsf{T}}) \right],
\end{align}
where $\hat{M}_{ij} = sinkhorn(M_{ij}, \tau)$, and $\gamma$ is an amplification factor. See algorithm in appendix for full details.

\section{Experiments}

\subsection{Datasets}

\textbf{The retinal fundus segmentation datasets} comprise five public datasets from different medical centers, denoted as Site A (RIM-ONE~\cite{fumero2011rim}), B (REFUGE~\cite{orlando2020refuge}), C (ORIGA~\cite{zhang2010origa}), D (REFUGE-Test~\cite{orlando2020refuge}), and E (Drishti-GS~\cite{sivaswamy2014drishti}), with 159, 400, 650, 800, and 101 images respectively, all consistently annotated for optic disc (OD) and optic cup (OC) segmentation. We adopted the preprocessing method outlined in~\cite{liu2022single,chen2024each}, where each image's region of interest is cropped to $800\times800$ pixels and normalized using min-max normalization.

\noindent \textbf{The polyp segmentation datasets} consist of four public datasets collected from different medical centers, denoted as Site A (BKAI-IGH-NEOPolyp~\cite{ngoc2021neounet}), B (CVC-ClinicDB/CVC-612~\cite{bernal2015wm}), C (ETIS~\cite{silva2014toward}), and D (Kvasir~\cite{jha2020kvasir}), containing 1000, 612, 196, and 1000 images, respectively. We followed the preprocessing steps in~\cite{chen2024each}, resizing the images to $800 \times 800$ pixels and normalizing them using ImageNet-derived statistics.

\subsection{Experimental Setup}
\label{exp_setup}
\textbf{Implementation Details.} To ensure a fair comparison across all methods, we followed the protocol in~\cite{chen2024each} and split each dataset into an $8:2$ ratio for training and testing. We employed the ResNet-50~\cite{he2016deep} pre-trained on ImageNet as the feature extractor. The training was optimized using the SGD optimizer with a momentum of 0.9 and a learning rate of 0.001. For source model training, we set the mini-batch size to 8, allowing simultaneous matching of 8 graphs per batch. The universe embedding $\mathcal{U}$ was treated as an additional trainable tensor, integrated into the network weights without influencing subsequent experiments with other methods.
During the TTA phase, all methods were trained on the target data without access to their labels. The mini-batch size was set to 4, and the previously learned $\mathcal{U}$ from source model training was frozen. All experiments were implemented using PyTorch and conducted on 4 NVIDIA 3090 GPUs.

\noindent \textbf{Evaluation Metrics.} The Dice score (DSC, \%) was used to quantify the accuracy of the predicted masks, serving as the primary metric for evaluating segmentation performance. Additionally, the enhanced alignment metric $E_\phi^{max}$~\cite{fan2018enhanced} was adopted to measure both pixel-level and global-level similarity. To further assess the consistency between predictions and ground truths, we also employed the structural similarity metric $S_{\alpha}$~\cite{fan2017structure}.

\subsection{Experimental Results}
We selected U-Net~\cite{ronneberger2015u} as the benchmark model for the segmentation task, training it on the source domains and testing on the target domain without adaptation (\textit{No Adapt}). In addition, we compared eight state-of-the-art (SOTA) methods, including two entropy-based approaches (TENT~\cite{wangtent} and SAR~\cite{niu2023towards}), a shape template-based method (TASD~\cite{liu2022single}), a dynamically adjusted learning rate method (DLTTA~\cite{yang2022dltta}), batch normalization-based methods (DomainAdaptor~\cite{zhang2023domainadaptor}, VPTTA~\cite{chen2024each}), and noise estimation-based approaches (DeY-Net~\cite{wen2024denoising} and NC-TTT~\cite{osowiechi2024nc}).

\begin{table*}[h!]
\centering
  \caption{Multi sources domain generaliation in retinal fundus segmentation. The average performance (mean $\pm$ standard deviation) of three trials for our method and nine SOTA methods. ``Site A'' means training on Sites B-E and testing on Site A, and similarly for the others. The best results are highlighted in \textcolor{red}{red}.}
  \vspace{-5pt}
  \begin{adjustbox}{width=0.99\linewidth}
    \begin{tabular}{c|ccccccccccccccc|ccc}
\hlineB{3}
\multirow{2}{*}{Methods} & \multicolumn{3}{c}{Site A} & \multicolumn{3}{c}{Site B} & \multicolumn{3}{c}{Site C} & \multicolumn{3}{c}{Site D} & \multicolumn{3}{c|}{Site E} & \multicolumn{3}{c}{Average} \\ \cline{2-19} 
                         & \textit{DSC}       & $E_\phi^{max}$       & $S_{\alpha}$      & \textit{DSC}       & $E_\phi^{max}$       & $S_{\alpha}$     & \textit{DSC}       & $E_\phi^{max}$       & $S_{\alpha}$    & \textit{DSC}       & $E_\phi^{max}$       & $S_{\alpha}$  & \textit{DSC}       & $E_\phi^{max}$       & $S_{\alpha}$      & \textit{DSC} $\uparrow$      & $E_\phi^{max}$ $\uparrow$       & $S_{\alpha}$ $\uparrow$    \\ \hline \hline
\textit{No Adapt} (U-Net~\cite{ronneberger2015u}) & 65.60\small{$\pm$5.78} & 91.75\small{$\pm$0.13} & 79.11\small{$\pm$0.07} & 74.65\small{$\pm$4.88} & 88.78\small{$\pm$0.08} & 83.31\small{$\pm$0.02} & 63.15\small{$\pm$7.30} & 81.03\small{$\pm$0.19} & 79.77\small{$\pm$0.03} & 68.11\small{$\pm$5.49} & 89.46\small{$\pm$0.20} & 82.12\small{$\pm$0.07} & 75.34\small{$\pm$1.01} & 90.68\small{$\pm$0.16} & 87.53\small{$\pm$0.05} & 69.37 & 88.34 & 82.36 \\ \hline
TENT (ICLR'21)~\cite{wangtent} & 75.22\small{$\pm$3.99} & \textcolor{red}{94.10\small{$\pm$0.10}} & 83.97\small{$\pm$0.05} & 81.06\small{$\pm$0.93} & 93.45\small{$\pm$0.11} & \textcolor{red}{92.36\small{$\pm$0.03}} &  75.23\small{$\pm$2.33} & 88.10\small{$\pm$0.14} & 84.62\small{$\pm$0.03} & 78.12\small{$\pm$2.59} & 91.33\small{$\pm$0.31} & 86.41\small{$\pm$0.12}  & 83.17\small{$\pm$0.78} & 92.74\small{$\pm$0.10} & 93.64\small{$\pm$0.03} & 78.56 & 91.94 & 88.20     \\
TASD (AAAI'22)~\cite{liu2022single} & 79.70\small{$\pm$2.52} & 88.73\small{$\pm$0.45} & 83.67\small{$\pm$0.17} & 82.96\small{$\pm$2.14} & 94.01\small{$\pm$0.20} & 90.99\small{$\pm$0.11} & 86.12\small{$\pm$1.11} & 93.59\small{$\pm$0.13} & 87.08\small{$\pm$0.07} & 78.02\small{$\pm$6.30} & 87.20\small{$\pm$0.16} & 87.94\small{$\pm$0.04} & 83.71\small{$\pm$3.77} & 90.88\small{$\pm$0.25} & 90.37\small{$\pm$0.12} & 82.10 & 90.88 & 88.01 \\

DLTTA (TMI'22)~\cite{yang2022dltta} & 71.88\small{$\pm$5.59} & 86.29\small{$\pm$0.21} & 80.70\small{$\pm$0.13} & 80.12\small{$\pm$2.26} & 90.53\small{$\pm$0.17} & 88.63\small{$\pm$0.06} & 81.46\small{$\pm$3.35} & 93.08\small{$\pm$0.08} & 87.48\small{$\pm$0.02} & 75.90\small{$\pm$7.27} & 84.27\small{$\pm$0.12} & 84.79\small{$\pm$0.06} & 82.75\small{$\pm$6.55} & 90.10\small{$\pm$0.31} & 89.33\small{$\pm$0.14} & 78.42 & 88.85 & 86.18 \\

SAR (ICLR'23)~\cite{niu2023towards} & 75.29\small{$\pm$2.69} & 88.73\small{$\pm$0.34} & 80.15\small{$\pm$0.26} & 86.42\small{$\pm$1.57} & 96.88\small{$\pm$0.27} & 91.28\small{$\pm$0.09} & 80.69\small{$\pm$4.30} & 92.10\small{$\pm$0.15} & 86.31\small{$\pm$0.27} & 80.27\small{$\pm$5.47} & 89.98\small{$\pm$0.31} & 87.20\small{$\pm$0.15} & 89.61\small{$\pm$0.28} & 95.90\small{$\pm$0.08} & 92.06\small{$\pm$0.01} & 82.46 & 92.71 & 87.40\\

DomainAdaptor (CVPR'23)~\cite{zhang2023domainadaptor} & 77.14\small{$\pm$1.21} & 90.80\small{$\pm$0.19} & 81.31\small{$\pm$0.02} & 83.95\small{$\pm$0.80} & 93.17\small{$\pm$0.09} & 90.44\small{$\pm$0.08} & 82.79\small{$\pm$1.09} & 95.36\small{$\pm$0.10} & 88.57\small{$\pm$0.13} & 86.50\small{$\pm$3.29} & 95.26\small{$\pm$0.33} & 91.07\small{$\pm$0.20} & 86.24\small{$\pm$1.44} & 94.90\small{$\pm$0.29} & 93.77\small{$\pm$0.18} & 83.32 & 93.90 & 89.03 \\

DeY-Net (WACV'24)~\cite{wen2024denoising} 
 & 79.67\small{$\pm$1.35} &89.14\small{$\pm$0.23} & 83.68\small{$\pm$0.04} & \textcolor{red}{90.07\small{$\pm$3.12}} & \textcolor{red}{97.06\small{$\pm$0.33}} & 92.31\small{$\pm$0.07} & 85.94\small{$\pm$1.85} & 93.54\small{$\pm$0.22} & 87.15\small{$\pm$0.10} & 83.47\small{$\pm$1.40} & 95.29\small{$\pm$0.17} & 88.70\small{$\pm$0.04} & 88.76\small{$\pm$2.28} & 95.43\small{$\pm$0.21} & 90.45\small{$\pm$0.08} & 85.58 & 94.09 & 88.46 \\

VPTTA (CVPR'24)~\cite{chen2024each} & 77.85\small{$\pm$1.40} & 94.09\small{$\pm$0.01} & 83.97\small{$\pm$0.05} & 84.57\small{$\pm$0.99} & 95.38\small{$\pm$0.03} & 89.53\small{$\pm$0.09} & 83.61\small{$\pm$1.67} & \textcolor{red}{96.39\small{$\pm$0.04}} & 89.64\small{$\pm$0.05} & 82.75\small{$\pm$0.50} & 96.53\small{$\pm$0.04} & 89.83\small{$\pm$0.01} & 89.74\small{$\pm$1.51} & \textcolor{red}{98.22\small{$\pm$0.24}} & \textcolor{red}{94.77\small{$\pm$0.07}} & 83.70 & \textcolor{red}{96.12} & 89.55 \\ 

NC-TTT (CVPR'24)~\cite{osowiechi2024nc} 
 & 80.83\small{$\pm$1.77} & 89.61\small{$\pm$0.23} & 84.41\small{$\pm$0.08} &87.07\small{$\pm$1.31} & 95.63\small{$\pm$0.06} & 90.09\small{$\pm$0.02} & 86.93\small{$\pm$2.27} & 94.69\small{$\pm$0.11} & 88.39\small{$\pm$0.03} & 82.10\small{$\pm$4.28} & 91.49\small{$\pm$0.22} & 86.33\small{$\pm$0.07} & 90.08\small{$\pm$0.99} & 95.93\small{$\pm$0.10} & 91.53\small{$\pm$0.09} & 85.40 &  93.47 & 88.15 \\  

\hline

Ours   &  \textcolor{red}{82.53\small{$\pm$1.52}}  &   91.85\small{$\pm$0.22}      &   \textcolor{red}{85.50\small{$\pm$0.20}}     &    88.98\small{$\pm$0.89}       &     96.69\small{$\pm$0.04}    &   91.63\small{$\pm$0.12}     &  \textcolor{red}{88.73\small{$\pm$0.70}}         &   95.80\small{$\pm$0.06}      &   \textcolor{red}{89.71\small{$\pm$0.02}}     &    \textcolor{red}{90.25\small{$\pm$1.33}}       &    \textcolor{red}{97.97\small{$\pm$0.19}}     &   \textcolor{red}{92.89\small{$\pm$0.13}}      &     \textcolor{red}{91.83 \small{$\pm$0.71}}     &  97.14\small{$\pm$0.20}      &   92.82\small{$\pm$0.09}  &  \textcolor{red}{88.46} & 95.84 &  \textcolor{red}{90.51} \\ 

\hlineB{3}
\end{tabular}
  \end{adjustbox}
  \label{tab:tabel_fundus}
\end{table*}

\begin{table*}[h!]
\centering
  \caption{Multi sources domain generaliation in the polyp segmentation. The average performance (mean $\pm$ standard deviation) of three trials for our method and nine SOTA methods. ``Site A'' means training on Sites B-D and testing on Site A, and similarly for the others. The best results are highlighted in \textcolor{red}{red}.}
  \vspace{-9pt}
  \begin{adjustbox}{width=0.99\linewidth}
    \begin{tabular}{c|cccccccccccc|ccc}
\hlineB{3}
\multirow{2}{*}{Methods} & \multicolumn{3}{c}{Site A} & \multicolumn{3}{c}{Site B} & \multicolumn{3}{c}{Site C} & \multicolumn{3}{c|}{Site D} & \multicolumn{3}{c}{Average} \\ \cline{2-16} 
                         & \textit{DSC}       & $E_\phi^{max}$       & $S_{\alpha}$      & \textit{DSC}       & $E_\phi^{max}$       & $S_{\alpha}$     & \textit{DSC}       & $E_\phi^{max}$       & $S_{\alpha}$    & \textit{DSC}       & $E_\phi^{max}$       & $S_{\alpha}$ & \textit{DSC} $\uparrow$      & $E_\phi^{max}$ $\uparrow$      & $S_{\alpha}$ $\uparrow$    \\ \hline \hline

\textit{No Adapt} (U-Net~\cite{ronneberger2015u}) & 82.67\small{$\pm$2.24} & 92.11\small{$\pm$0.07} & 89.59\small{$\pm$0.10} &  81.20\small{$\pm$4.38} & 92.46\small{$\pm$0.20} & 86.08\small{$\pm$0.06} & 78.33\small{$\pm$1.62} & 94.00\small{$\pm$0.08} & 85.35\small{$\pm$0.01} &  75.51\small{$\pm$2.74} & 85.90\small{$\pm$0.17} & 81.71\small{$\pm$0.21} &  79.43 & 91.12 & 85.68    \\ \hline
TENT (ICLR'21)~\cite{wangtent} & 78.69\small{$\pm$1.12} & 88.28\small{$\pm$0.14} & 85.44\small{$\pm$0.08} & 77.71\small{$\pm$3.09} & 89.73\small{$\pm$0.09} & 81.90\small{$\pm$0.12} & 80.19\small{$\pm$3.35} & 96.03\small{$\pm$0.04} & 87.80\small{$\pm$0.17} & 72.63\small{$\pm$5.11} & 82.47\small{$\pm$0.30} & 78.92\small{$\pm$0.14} & 77.31 &  89.13 &   83.52          \\
TASD (AAAI'22)~\cite{liu2022single} & 84.60\small{$\pm$4.14} & 92.55\small{$\pm$0.21} & 90.76\small{$\pm$0.10} & 85.28\small{$\pm$2.20} & 93.10\small{$\pm$0.13} & 86.83\small{$\pm$0.19} & 82.43\small{$\pm$5.39} & 96.31\small{$\pm$0.08} & 89.61\small{$\pm$0.02} & 80.04\small{$\pm$4.41} & 88.26\small{$\pm$0.11} & 84.09\small{$\pm$0.03} & 83.09 & 92.56 & 87.82\\

DLTTA (TMI'22)~\cite{yang2022dltta} & 80.68\small{$\pm$2.37} & 90.01\small{$\pm$0.21} & 86.89\small{$\pm$0.09} & 80.86\small{$\pm$1.77} & 91.98\small{$\pm$0.10} & 85.83\small{$\pm$0.04} & 75.57\small{$\pm$3.12} & 93.90\small{$\pm$0.08} & 84.15\small{$\pm$0.10} & 74.25\small{$\pm$3.40} & 83.49\small{$\pm$0.12} & 79.44\small{$\pm$0.10} & 77.84 & 89.85 & 84.08 \\

SAR (ICLR'23)~\cite{niu2023towards} & 79.37\small{$\pm$0.77} & 88.40\small{$\pm$0.12} & 86.11\small{$\pm$0.03} & 78.40\small{$\pm$2.07} & 90.23\small{$\pm$0.11} & 83.39\small{$\pm$0.02} & 80.94\small{$\pm$0.90} & 95.83\small{$\pm$0.12} & 86.84\small{$\pm$0.07} & 78.23\small{$\pm$2.91} & 86.37\small{$\pm$0.09} & 81.66\small{$\pm$0.14} & 79.24 & 90.21 & 84.50\\

DomainAdaptor (CVPR'23)~\cite{zhang2023domainadaptor} & 88.58\small{$\pm$0.96} & 95.03\small{$\pm$0.04} & 90.95\small{$\pm$0.01} & 81.12\small{$\pm$1.07} & 92.31\small{$\pm$0.08} & 86.60\small{$\pm$0.02} & 81.77\small{$\pm$2.17} & 95.18\small{$\pm$0.03} &  87.19\small{$\pm$0.01} & 79.91\small{$\pm$1.33} & 87.23\small{$\pm$0.04} & 83.64\small{$\pm$0.05} & 82.85 & 92.44 & 87.10 \\

DeY-Net (WACV'24)~\cite{wen2024denoising} & 83.61\small{$\pm$2.31} & 90.79\small{$\pm$0.11} & 89.25\small{$\pm$0.09} & 80.90\small{$\pm$3.17} & 91.49\small{$\pm$0.10} & 85.36\small{$\pm$0.03} & 80.46\small{$\pm$1.43} & 95.83\small{$\pm$0.08} & 87.14\small{$\pm$0.03} & 73.85\small{$\pm$3.04} & 83.00\small{$\pm$0.13} & 78.96\small{$\pm$0.08} & 79.71 & 90.28 & 85.18\\

VPTTA (CVPR'24)~\cite{chen2024each} & 82.34\small{$\pm$0.69} & 91.03\small{$\pm$0.08} & 88.35\small{$\pm$0.01} & 84.32\small{$\pm$0.44} & 91.53\small{$\pm$0.07} & 87.46\small{$\pm$0.02} & 84.43\small{$\pm$0.52} & 95.75\small{$\pm$0.16} & 88.91\small{$\pm$0.03} & 82.40\small{$\pm$0.29} & 89.60\small{$\pm$0.04} & 85.18\small{$\pm$0.01} & 83.37 & 91.97 & 87.47\\ 

NC-TTT (CVPR'24)~\cite{osowiechi2024nc} & 87.96\small{$\pm$1.22} & 93.71\small{$\pm$0.13} & 90.66\small{$\pm$0.07} & 83.70\small{$\pm$2.08} & 93.78\small{$\pm$0.09} & 87.51\small{$\pm$0.07} & \textcolor{red}{89.29\small{$\pm$0.90}} & 96.66\small{$\pm$0.22} & \textcolor{red}{91.26\small{$\pm$0.07}} & 78.23\small{$\pm$0.79} & 86.37\small{$\pm$0.09} & 81.66\small{$\pm$0.06} & 84.80 & 92.63 & 87.77 \\

\hline
Ours  &   \textcolor{red}{90.97\small{$\pm$0.66}}  &  \textcolor{red}{96.57\small{$\pm$0.06}}   & \textcolor{red}{92.92\small{$\pm$0.08}}  & \textcolor{red}{87.49\small{$\pm$0.82}} & \textcolor{red}{95.21\small{$\pm$0.04}}  & \textcolor{red}{89.93\small{$\pm$0.09}} & 86.87\small{$\pm$0.58} 
 & \textcolor{red}{97.53\small{$\pm$0.05}} &  90.16\small{$\pm$0.12} & \textcolor{red}{83.97\small{$\pm$1.18}} & \textcolor{red}{90.67\small{$\pm$0.13}} & \textcolor{red}{86.44\small{$\pm$0.08}} &  \textcolor{red}{87.32} &   \textcolor{red}{94.67}     &    \textcolor{red}{89.86}    \\ 
\hlineB{3}
\end{tabular}
  \end{adjustbox}
  \label{tab:tabel_polyp}
\end{table*}

\noindent \textbf{Multi-Source Generalization.} In these experiments, we adopted a leave-one-out training strategy~\cite{chen2023improved} (i.e., with $S = |\mathcal{D}_s \cup \mathcal{D}_t|-1$ and $T = 1$). Although some methods, such as TASD~\cite{liu2022single}, DeY-Net~\cite{wen2024denoising} and VPTTA~\cite{chen2024each}, were originally designed for single-source domain training, we simulated single-domain conditions by mixing multi-source domain data, making the experimental setup still feasible for these approaches. Furthermore, we have also included experiments specifically focused on single-source domains later in the study. The segmentation results for the fundus and polyp datasets are shown in Tables~\ref{tab:tabel_fundus} and \ref{tab:tabel_polyp}, respectively. For OD/OC segmentation, all TTA methods outperformed the \textit{No Adapt} baseline, highlighting their effectiveness in addressing domain shifts. Comparatively, our method consistently achieved better average performance than all other approaches. Specifically, for the key segmentation metric DSC, our method exceeded the second-best (DeY-Net~\cite{wen2024denoising}) by 2.88\% and outperformed the \textit{No Adapt} by 19.09\%. A similar trend was observed in the polyp segmentation results, where we outperformed the second-best method (NC-TTT~\cite{osowiechi2024nc}) and \textit{No Adapt} by 2.52\% and 7.89\%, respectively. Additionally, our approach exhibited greater stability across multiple trials compared to the other methods.

\begin{table*}[h!]
\centering
  \caption{Single source domain generalization in the polyp segmentation. The average performance of three trials for our method and nine SOTA methods. A $\rightarrow$ B represents models trained on Site A and tested on Site B, and similar for others. Best results are colored as \textcolor{red}{red}.}
  \vspace{-10pt}
  \begin{adjustbox}{width=0.99\linewidth}
    \begin{tabular}{c|cccccccccccc|c}
\hlineB{3}
\multirow{2}{*}{Methods}      & A $\rightarrow$ B & A $\rightarrow$ C & A $\rightarrow$ D & B $\rightarrow$ A & B $\rightarrow$ C & B $\rightarrow$ D & C $\rightarrow$ A & C $\rightarrow$ B & C $\rightarrow$ D & D $\rightarrow$ A & D $\rightarrow$ B & D $\rightarrow$ C & Avg. \\ \cline{2-14} 
& \multicolumn{13}{c}{\textbf{Dice Score Metric $\uparrow$ (DSC, mean$\pm$std )}} \\ \hline \hline
\textit{No Adapt} (U-Net~\cite{ronneberger2015u})  & 76.49\small{$\pm$1.48} & 67.70\small{$\pm$3.27} &  70.94\small{$\pm$7.56} & 76.21\small{$\pm$2.90} & 62.80\small{$\pm$8.03} & 70.03\small{$\pm$4.53} & 72.87\small{$\pm$3.62} & 69.13\small{$\pm$5.15} & 70.44\small{$\pm$3.27} & 79.78\small{$\pm$1.08} & 72.47\small{$\pm$2.80} & 74.91\small{$\pm$3.18} & 71.98  \\ \hline
TENT (ICLR'21)~\cite{wangtent} & 68.13\small{$\pm$0.18} & 65.99\small{$\pm$0.11} & 66.11\small{$\pm$0.14} & 77.56\small{$\pm$0.21}  & 59.87\small{$\pm$0.20} & 68.89\small{$\pm$0.14} & 73.21\small{$\pm$0.10} & 65.04\small{$\pm$0.09} & 72.35\small{$\pm$0.22} & 77.88\small{$\pm$0.14} &  70.01\small{$\pm$0.11} &  73.60\small{$\pm$0.20}  &69.89    \\
TASD (AAAI'22)~\cite{liu2022single} & 73.02\small{$\pm$2.24} & 79.67\small{$\pm$6.37} & 78.23\small{$\pm$4.44} & 66.82\small{$\pm$9.17} & 72.09\small{$\pm$4.40} & 77.95\small{$\pm$1.69} & 74.20\small{$\pm$3.30} & 69.74\small{$\pm$7.27} & 76.13\small{$\pm$2.82} & 83.77\small{$\pm$0.51} & 76.09\small{$\pm$3.26} & 80.10\small{$\pm$1.47} & 75.65\\

DLTTA (TMI'22)~\cite{yang2022dltta} & 72.23\small{$\pm$0.04} & 70.78\small{$\pm$0.06} & 74.56\small{$\pm$0.04} & 64.87\small{$\pm$0.01} & 60.05\small{$\pm$0.09} & 78.47\small{$\pm$0.03} & 71.22\small{$\pm$0.03} & 65.78\small{$\pm$0.10} & 71.00\small{$\pm$0.13} & 79.89\small{$\pm$0.08} & 79.10\small{$\pm$0.10} & 74.52\small{$\pm$0.01} & 71.87 \\

SAR (ICLR'23)~\cite{niu2023towards} & 70.34\small{$\pm$0.13} & 77.90\small{$\pm$0.20} & 77.71\small{$\pm$0.09} & 68.19\small{$\pm$0.12} & 66.56\small{$\pm$0.08} & 73.92\small{$\pm$0.13} & 76.12\small{$\pm$0.11} & 71.29\small{$\pm$0.19} & 75.08\small{$\pm$0.08} & 82.07\small{$\pm$0.04} & 80.44\small{$\pm$0.21} & 83.10\small{$\pm$0.15} & 75.22 \\

DomainAdaptor (CVPR'23)~\cite{zhang2023domainadaptor} & 78.39\small{$\pm$0.61} & 77.09\small{$\pm$0.35} & 79.45\small{$\pm$0.21} & 74.98\small{$\pm$0.70} & 70.07\small{$\pm$0.66} & 80.21\small{$\pm$0.46} & 76.02\small{$\pm$0.33} & 73.40\small{$\pm$0.50} & \textcolor{red}{81.33\small{$\pm$0.54}} & 86.10\small{$\pm$0.43} & 80.56\small{$\pm$0.62} & 77.86\small{$\pm$0.21} & 77.96 \\

DeY-Net (WACV'24)~\cite{wen2024denoising} & 74.45\small{$\pm$1.31} & 76.07\small{$\pm$2.52} & 80.11\small{$\pm$0.72} & 70.90\small{$\pm$1.10} & 69.23\small{$\pm$2.18} & 76.31\small{$\pm$1.32} & 78.33\small{$\pm$4.43} & 74.60\small{$\pm$1.66} &73.52\small{$\pm$2.25} & 80.70\small{$\pm$0.27} & 78.23\small{$\pm$1.17} & 81.48\small{$\pm$0.93} & 76.16 \\

VPTTA (CVPR'24)~\cite{chen2024each}       &78.18\small{$\pm$0.01}  &  75.14\small{$\pm$0.03}  &  82.73\small{$\pm$0.02}  & 73.65\small{$\pm$0.06} &   64.78\small{$\pm$0.01}     &    82.86\small{$\pm$0.03}    &  77.90\small{$\pm$0.01}      &   70.64\small{$\pm$0.05}     &   77.22\small{$\pm$0.10}     &   87.19\small{$\pm$0.03}     &    \textcolor{red}{84.97\small{$\pm$0.08}}    &   82.12\small{$\pm$0.02}     &  78.12  \\ 

NC-TTT (CVPR'24)~\cite{osowiechi2024nc} & 79.21\small{$\pm$0.24} & 81.76\small{$\pm$0.57} & \textcolor{red}{83.22\small{$\pm$0.69}} & 77.35\small{$\pm$0.33} & 79.40\small{$\pm$0.67} & 80.94\small{$\pm$0.54} & \textcolor{red}{80.17\small{$\pm$0.70}} & 72.36\small{$\pm$0.62} & 78.55\small{$\pm$0.43} & 84.09\small{$\pm$0.38} & 82.38\small{$\pm$0.51} & 79.77\small{$\pm$0.74} & 79.93 \\


\hline

Ours  & \textcolor{red}{80.47\small{$\pm$0.28}} &    \textcolor{red}{83.93\small{$\pm$0.19}} & 83.18\small{$\pm$0.33} & \textcolor{red}{78.33\small{$\pm$0.26}}  &  \textcolor{red}{81.25\small{$\pm$0.23}}  & \textcolor{red}{83.58\small{$\pm$0.34}}  &  78.86\small{$\pm$0.26} &    \textcolor{red}{76.49\small{$\pm$0.21}} &    80.20\small{$\pm$0.24}    &    \textcolor{red}{89.17\small{$\pm$0.22}}    &    81.68\small{$\pm$0.22}    &   \textcolor{red}{83.86\small{$\pm$0.28}}     &  \textcolor{red}{81.75}    \\ \hlineB{3}
\end{tabular}
  \end{adjustbox}
  \label{tab:tabel_single}
\end{table*}
    
\noindent \textbf{Single-Source Generalization.} In these experiments, the models are trained on one domain and tested on the remaining datasets (i.e., with $S=1$ and $T=|\mathcal{D}_s \cup \mathcal{D}_t|-1$). Compared to multi-source generalization, single-source training is considered more challenging due to the limited domain information available during the training phase. By comparing Tables~\ref{tab:tabel_polyp} and \ref{tab:tabel_single}, we can observe that all models show lower segmentation performance (in terms of DSC) in the single-source setting compared to the multi-source results. Some methods even perform worse than the \textit{No Adapt} baseline in single-source scenarios.
However, our method still achieves the best average performance across 12 domain transfer experiments, outperforming the second-best approach (NC-TTT~\cite{osowiechi2024nc}) by 1.83\% and exceeding the \textit{No Adapt} by 9.77\% in terms of DSC. 

Fig.~\ref{fig:visual} presents the visualization results for ``Site A'' in Table~\ref{tab:tabel_fundus}. The \textit{No Adapt} baseline exhibits significant misalignment, with notably distorted shapes and inaccurate boundaries. VPTTA offers some improvement but still suffers from distortions and incomplete segmentation, particularly in the optic cup area. NC-TTT shows instances of overlapping segmentation. In contrast, our method delivers the most precise segmentation, closely aligning with the ground truth. The Grad-CAM~\cite{selvaraju2017grad} visualizations reinforce this, as the attention of network is focused specifically on the relevant regions. These results demonstrate our method's enhanced generalization to unseen domains by effectively incorporating morphological priors, resulting in higher segmentation accuracy. For further visualization and additional experimental results, please refer to the appendix.

\begin{figure}[!t]
    \centering
\includegraphics[width=0.999\linewidth]{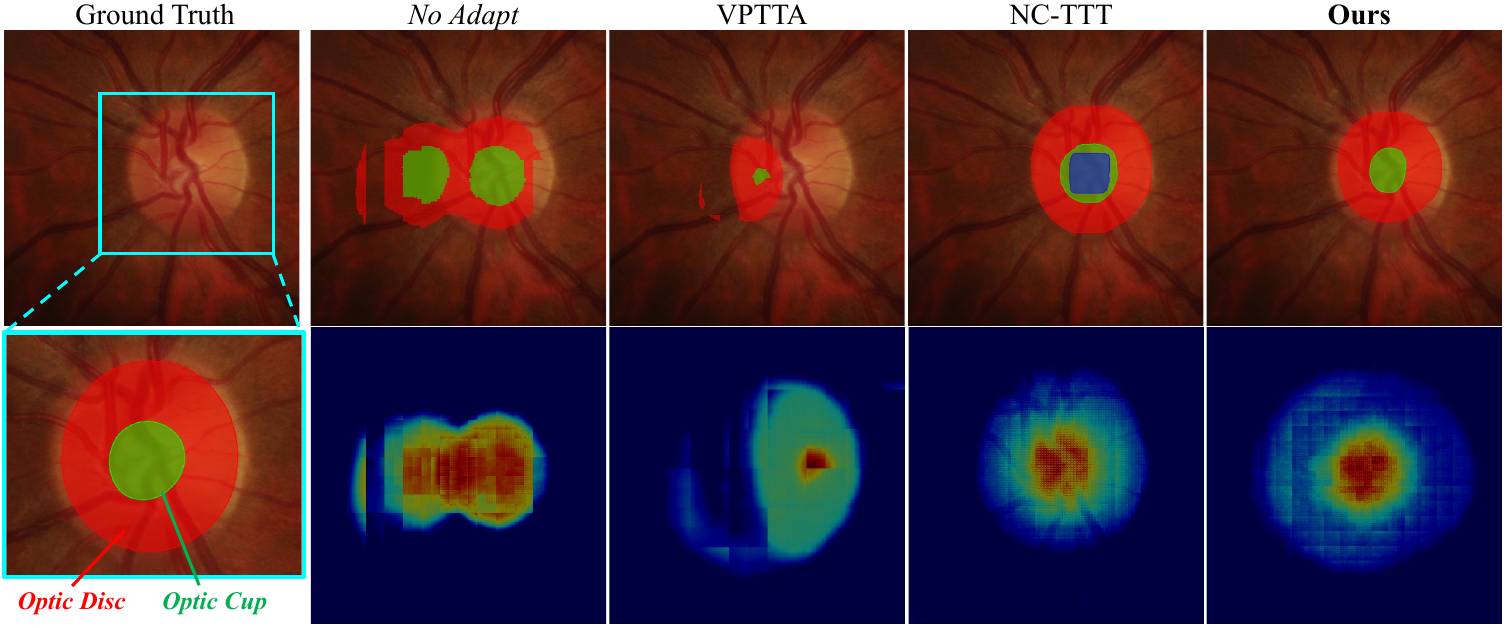}
    \caption{Visualization comparison of segmentation results and Grad-CAM outputs from the final layer of the backbone network for the \textit{No Adapt} baseline, VPTTA~\cite{chen2024each}, NC-TTT~\cite{osowiechi2024nc}, and our proposed method on retinal fundus images. Additional visual comparisons are provided in the supplementary material.}
     \vspace{-0.5cm}
    \label{fig:visual}
\end{figure}
\section{Analysis}

\begin{table}[htb!]
\setlength{\tabcolsep}{6pt}
\centering
  \caption{Ablation study in the retinal fundus segmentation. Details of the experiment can be found in Sec~\ref{sec:abl_ue}.}
  \vspace{-0.3cm}
  \begin{adjustbox}{width=0.999\linewidth}
    \begin{tabular}{c|ccccc|c}
\hlineB{3}
\multirow{2}{*}{Methods}     & Site A & Site B & Site C & Site D & Site E & Average \\ \cline{2-7}
& \multicolumn{6}{c}{\textbf{Dice Score Metric $\uparrow$ (DSC, mean$\pm$std )}} \\ \hline \hline

\textit{No Adapt} & 65.60\small{$\pm$5.78} & 74.65\small{$\pm$4.88} & 63.15\small{$\pm$7.30} & 68.11\small{$\pm$5.49} & 75.34\small{$\pm$1.01} & 69.37\\ \hline

\textit{w/o $\mathcal{U}$}      &  77.09\small{$\pm$1.90}      & 83.24\small{$\pm$2.01}       &  80.50\small{$\pm$1.44}      &  78.32\small{$\pm$1.55}      &    83.18\small{$\pm$2.18}    &  80.47       \\
\textit{w/o priors} &   79.51\small{$\pm$0.80}     & 79.36\small{$\pm$1.13}       &  77.40\small{$\pm$1.89}      &    77.29\small{$\pm$1.77}    &  80.00\small{$\pm$1.93}      &    78.71     \\
\hline
Ours &  82.53\small{$\pm$1.52} &  88.98\small{$\pm$0.89}      &  88.73\small{$\pm$0.70}      &  90.25\small{$\pm$1.33}      &   91.83\small{$\pm$0.71}     &   \textbf{88.46}    \\
\hlineB{3}
\end{tabular}
  \end{adjustbox}
  \label{tab:tabel_abl1}
\end{table}


\subsection{Effectiveness of the Universe Embeddings}
\label{sec:abl_ue}
The universe embeddings $\mathcal{U}$ play a crucial role in incorporating morphological priors from medical images, resulting in an assignment matrix that projects each node to the \textit{universe of nodes}. To validate the effectiveness during the TTA, we conducted the following experiments: (1) Without universe embeddings (denoted as \textit{w/o $\mathcal{U}$}), the universe matching assignment matrix $\bold{U}$ is initialized following the setting in \cite{wang2020graduated} (i.e. $\bold{U}=1/d+10^{-3}z$, where $z\sim N(0,1)$), leading to random matching between the nodes and the universe; (2) Without morphological priors but with $\mathcal{U}$ (denoted as \textit{w/o priors}). As shown in Table~\ref{tab:tabel_abl1}, the performance in (1) is slightly better than in (2) by 1.76\% in terms of DSC on average. However, when using the pre-trained $\mathcal{U}$ derived from the source model, the segmentation results show a significant improvement. This highlights the effectiveness of the pre-trained morphological priors embedded in $\mathcal{U}$ for medical imaging tasks.

\begin{figure}[!t]
    \centering
\includegraphics[width=0.999\linewidth]{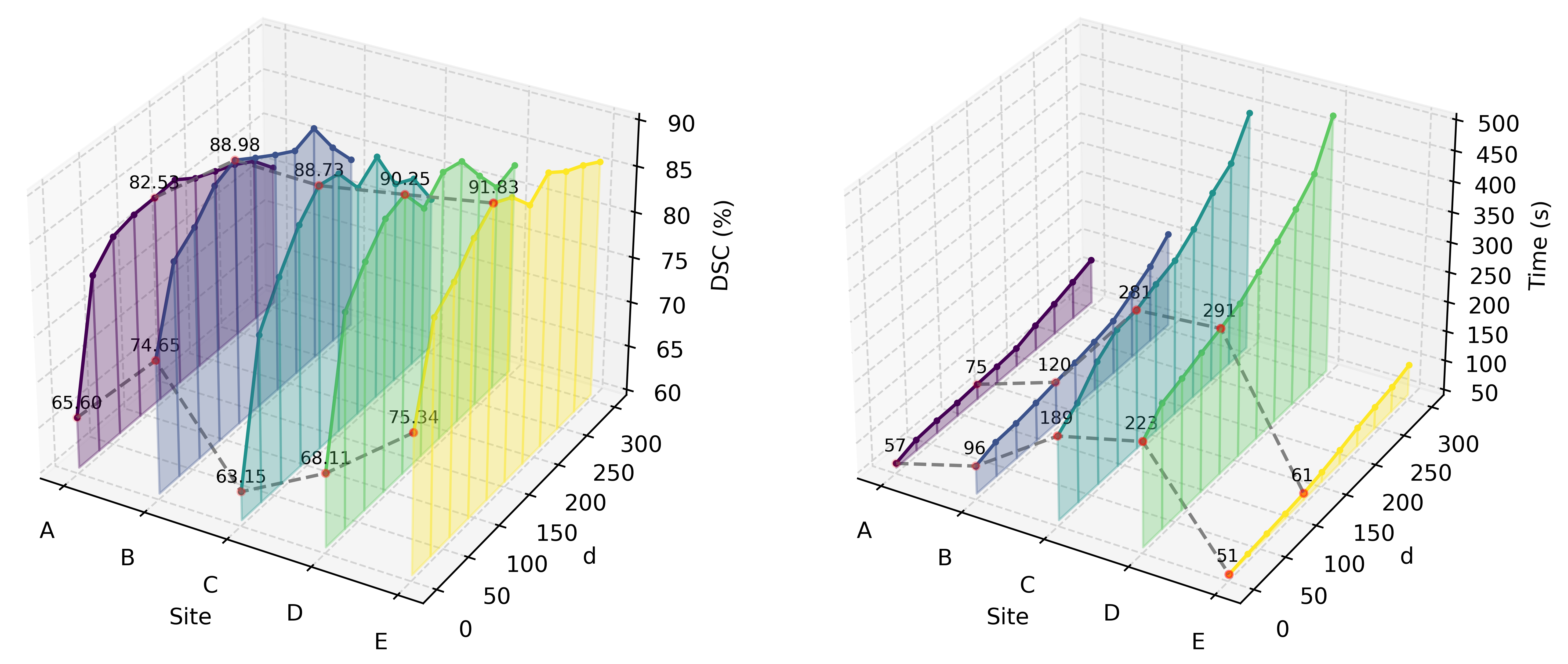}
    \caption{The performance of DSC (\%) and inference time (s) with different universe sizes $d$ are shown, with the experimental setup identical to that in Table~\ref{tab:tabel_fundus}. The first curve represents the \textit{No Adapt} baseline ($d = 0$), while the subsequent curve corresponds to the results obtained with a universe size of $ d = 120$.}
     \vspace{-0.5cm}
    \label{fig:universe_d}
\end{figure}

\vspace{-5pt}
\subsection{Effectiveness of the Universe Size}
The universe size $d$ directly impacts both the matching accuracy and computational efficiency, and its selection depends on the number of sampled nodes. Previous works~\cite{bernard2019hippi,wang2020graduated,nurlanov2023universe}, ensured that the $d$ matched the number of nodes in each graph by labeling an equal number of keypoints. In our experiments, $d$ is determined by the spatially-uniform sampling described in Sec~\ref{sec:graph_generation}, and we empirically set $d = 100\times (n+1) s^{-1}$, where $s$ is the sampling step and $n$ is the number of segmentation categories required for the task. We evaluated several $d$, with the results shown in Fig.~\ref{fig:universe_d}. We observe that excessively small or large step sizes—resulting in too many or too few sampled nodes—negatively impact both training speed and performance. However, when $d$ is within a reasonable range (in this experiment, $s \in[2, 10]$ i.e. $d \in [30, 150]$.), the model maintains good scalability and performance.

\begin{table}[htb!]
\vspace{-9pt}
\setlength{\tabcolsep}{8pt}
\centering
  \caption{Comparison of common variables and performance between pairwise matching and multi-matching, using the same experimental setup as in Table~\ref{tab:tabel_fundus}. Both methods follow the graph construction in Sec.~\ref{sec:graph_generation}.}
  \vspace{-0.3cm}
  \begin{adjustbox}{width=0.979\linewidth}
    \begin{tabular}{r|cc|l}
\hlineB{3}
\multirow{2}{*}{param} & \multicolumn{2}{c|}{Site A-E Average}              & \multirow{2}{*}{description}      \\ \cline{2-3}
&Pairwise~\cite{wang2019learning}         & Multi-Matching (Ours) & \\ \hline \hline
lr                     & $10^{-3}$ & $10^{-3}$ & learning rate \\
batch & 4 & 4  & batch size in  inference  \\
$h$ & 256 & 256  & dimension of node feature\\
$\tau$  & 0.05   & 0.05  & regularization factor of Sinkhorn \\
\textit{Iter} & 20 & 20 & max iterations of Sinkhorn        \\
$d$    & -  & 120 & universe size \\ \hline
DSC $\uparrow$  &  84.63   &  \textbf{88.46}  & dice score metric (\%) \\
time $\downarrow$& 0.930 &  \textbf{0.392}   & inference time per image (s/img)  \\
Param $\downarrow$ & 1.071 & \textbf{0.658} & parameter count (M) \\
FLOPs $\downarrow$&  15.35  &  \textbf{4.255}  &   floating point operations per second (G)  \\ \hlineB{3}
\end{tabular}

  \end{adjustbox}
  \label{tab:tabel_abl_matching}
\end{table}
\vspace{-0.5cm}
\subsection{Multi-Matching vs. Pairwise Matching}
Compared to pairwise matching, multi-matching uses cycle-consistency for global optimization, which helps avoid local optima in pairwise. Moreover, joint optimization reduces the computational complexity of large-scale matching. By incorporating morphological consistency constraints, multi-matching exhibits greater robustness. To validate these claims, we conducted experiments comparing the two approaches. 
For the pairwise matching method, we adopted the benchmark approach~\cite{wang2019learning} and applied the same graph construction process as in multi-matching. The comparison results are presented in Table~\ref{tab:tabel_abl_matching}. It is evident that multi-matching outperforms pairwise matching in terms of DSC, inference time, parameters and FLOPs.
The inference time of multi-matching was reduced by approximately 57.85\% compared to pairwise matching, while segmentation accuracy improved by 3.83\% in DSC.
\section{Conclusion}
This paper presents a novel multi-matching framework for TTDG, which leverages universe embeddings to incorporate morphological priors from medical images while ensuring cycle-consistency, along with an unsupervised test-time paradigm that fully integrates these priors for efficient adaptation.
Extensive experiments, including comparisons and ablation studies demonstrate that our graph-matching-based approach achieves SOTA performance, outperforming entropy-based, template-based, batch normalization-based, and noise estimation-based methods across both multi-source and single-source domain generalization tasks.

\newpage
\section*{Acknowledgements}

This work was supported in part by the National Natural Science Foundation of China (Grant No. 62306003), the Open Research Fund from Guangdong Laboratory of Artificial Intelligence and Digital Economy (SZ), under Grant No. GML-KF-24-29 and the Open Foundation of Jiangxi Provincial Key Laboratory of Image Processing and Pattern Recognition (ET202404437).
{
    \small
    \bibliographystyle{ieeenat_fullname}
    \bibliography{main}
    
}

\maketitlesupplementary

\newtheorem{lemma}{Lemma}
\setcounter{page}{1}
\setcounter{section}{0}
\setcounter{table}{0}
\setcounter{figure}{0}
\setcounter{equation}{0}

\renewcommand{\thefigure}{A\arabic{figure}} 
\renewcommand{\thetable}{A\arabic{table}} 
\renewcommand{\thesection}{A\arabic{section}} 
\renewcommand{\thesubsection}{A\arabic{subsection}} 


\section{Theoretical Analysis}
\begin{lemma}[Cycle-consistency, Universe Matching]
Given a set of pairwise (partial) matching matrices $\{\mathbf{X}_{ij}\}_{i,j=1}^m$, it is cycle-consistent iff there exists a collection of universe matching matrices \(\{ U_i \in \mathbb{U}_{n_i, d} \}_{i=1}^m\) such that for each graph pair $(\mathcal{G}_i, \mathcal{G}_j)$, we have
\begin{equation}
    \mathbf{X}_{ij} = U_i U_j^{\mathsf{T}}.\nonumber
\end{equation}
\end{lemma}

\begin{proof}
We need to prove that the matching matrices $\{\mathbf{X}_{ij}\}_{i,j=1}^m$ satisfy cycle-consistency, which means that:
\begin{equation}
    \mathbf{X}_{ij} \mathbf{X}_{jk} = \mathbf{X}_{ik}. \forall i, j, k\in [m]
\end{equation}
 
By the assumption of the \textit{lemma 1}, there exists a collection of universe matching matrices $\{ U_i \}_{i=1}^m$ such that for each pair $(\mathcal{G}_i, \mathcal{G}_j)$,
\begin{equation}
    \mathbf{X}_{ij} = U_i U_j^{\mathsf{T}}.
\end{equation}
Therefore, we have
\begin{equation}
    \mathbf{X}_{ij} = U_i U_j^{\mathsf{T}}, \quad \mathbf{X}_{jk} = U_j U_k^{\mathsf{T}}, \quad \mathbf{X}_{ik} = U_i U_k^{\mathsf{T}}.
\end{equation}
Now we compute $\mathbf{X}_{ij} \mathbf{X}_{jk}$:
\begin{align}
    \mathbf{X}_{ij} \mathbf{X}_{jk} &= (U_i U_j^{\mathsf{T}})(U_j U_k^{\mathsf{T}}) \nonumber \\ 
    &= U_i (U_j^{\mathsf{T}} U_j) U_k^{\mathsf{T}}. 
    \label{eq.6}
\end{align}

Assume that each matrix $U_i$ satisfies $U_i^{\mathsf{T}} U_i = I$ (i.e., $U_i$ is an orthogonal matrix or has unit inner product property). 
Thus, Eq. (\ref{eq.6}) simplifies further to:
\begin{equation}
    \mathbf{X}_{ij} \mathbf{X}_{jk} = U_i I U_k^{\mathsf{T}} = U_i U_k^{\mathsf{T}} = \mathbf{X}_{ik}.
\end{equation}

This shows that the matching matrices $\{\mathbf{X}_{ij}\}_{i,j=1}^m$ satisfy cycle-consistency. 
\end{proof}

\section{Algorithm Pipeline}
\label{sec:algo}
Algorithms~\ref{algo:train} and \ref{algo:test} outline the procedures for the source training phase and the test-time adaptation phase, respectively, while Algorithm~\ref{algo:hippi} provides a detailed explanation of the HiPPI~\cite{bernard2019hippi} method used in Algorithm~\ref{algo:train}.

\section{Additional Experiments}
\subsection{Single Source DG in Retinal Fundus}
\begin{table*}[h!]
\centering
  \caption{Single source domain generalization in the retinal fundus segmentation. The performance (mean $\pm$ standard deviation) of three trials for our method and eight SOTA methods. ``A $\rightarrow \{B,C,D,E\}$'' represents models trained on Site A and tested on the mixed distribution of Sites B-E, and similar for others. Best results are colored as \textcolor{red}{red}.}
  \begin{adjustbox}{width=0.99\linewidth}

\begin{tabular}{c|ccccccccccccccc|ccc}
\hlineB{3}
\multirow{2}{*}{Methods} & \multicolumn{3}{c}{A $\rightarrow \{B,C,D,E\}$} & \multicolumn{3}{c}{B $\rightarrow  \{A,C,D,E\}$} & \multicolumn{3}{c}{C $\rightarrow  \{A,B,D,E\}$} & \multicolumn{3}{c}{D $\rightarrow  \{A,B,C,E\}$} & \multicolumn{3}{c|}{E $\rightarrow  \{A,B,C,D\}$} & \multicolumn{3}{c}{Average} \\ \cline{2-19} 
                         & \textit{DSC}       & $E_\phi^{max}$       & $S_{\alpha}$      & \textit{DSC}       & $E_\phi^{max}$       & $S_{\alpha}$     & \textit{DSC}       & $E_\phi^{max}$       & $S_{\alpha}$    & \textit{DSC}       & $E_\phi^{max}$       & $S_{\alpha}$  & \textit{DSC}       & $E_\phi^{max}$       & $S_{\alpha}$      & \textit{DSC} $\uparrow$      & $E_\phi^{max}$ $\uparrow$       & $S_{\alpha}$ $\uparrow$    \\ \hline \hline
\textit{No Adapt} (U-Net~\cite{ronneberger2015u}) &70.60\small{$\pm$10.01} & 86.92\small{$\pm$1.17} & 80.36\small{$\pm$0.88} & 77.08\small{$\pm$6.90} & 91.58\small{$\pm$0.84} & 85.44\small{$\pm$0.99} & 66.24\small{$\pm$8.45} & 86.49\small{$\pm$0.77} & 80.01\small{$\pm$0.91} & 71.21\small{$\pm$9.47} & 82.90\small{$\pm$1.48} & 80.11\small{$\pm$0.88} & 72.26\small{$\pm$7.60} & 86.51\small{$\pm$1.02} & 86.05\small{$\pm$0.68} & 71.47 & 86.88 & 82.39\\ \hline
TASD (AAAI'22)~\cite{liu2022single} & 79.89\small{$\pm$5.91} & 93.26\small{$\pm$0.21} & 87.12\small{$\pm$0.13} & 82.63\small{$\pm$3.24} & 93.20\small{$\pm$0.25} & 86.00\small{$\pm$0.10} & 78.03\small{$\pm$4.29} & 92.47\small{$\pm$0.14} & 86.71\small{$\pm$0.09} & 76.30\small{$\pm$7.81} & 86.09\small{$\pm$0.15} & 80.94\small{$\pm$0.11} & 79.99\small{$\pm$1.29} & 93.24\small{$\pm$0.10} & 87.08\small{$\pm$0.08} & 79.36 & 91.65 & 85.57  \\

DLTTA (TMI'22)~\cite{yang2022dltta} & 74.96\small{$\pm$7.20} & 89.24\small{$\pm$0.25} & 84.02\small{$\pm$0.11} & 78.27\small{$\pm$5.66} & 92.40\small{$\pm$0.40} & 85.36\small{$\pm$0.11} & 75.84\small{$\pm$5.14} & 90.96\small{$\pm$0.20} & 84.11\small{$\pm$0.10} & 65.55\small{$\pm$9.35} & 84.80\small{$\pm$0.18} & 78.33\small{$\pm$0.08} & 71.68\small{$\pm$4.99} & 87.79\small{$\pm$0.17} & 85.13\small{$\pm$0.09} & 73.26 & 89.03 & 83.39\\

SAR (ICLR'23)~\cite{niu2023towards} & 74.20\small{$\pm$6.09} & 89.07\small{$\pm$0.33} & 83.64\small{$\pm$0.10} & 80.34\small{$\pm$2.86} & 92.70\small{$\pm$0.41} & 85.95\small{$\pm$0.12} & 72.58\small{$\pm$4.46} & 91.20\small{$\pm$0.20} & 83.47\small{$\pm$0.09} & 70.30\small{$\pm$8.98} & 85.10\small{$\pm$1.09} & 79.42\small{$\pm$0.72} & 70.31\small{$\pm$5.77} & 88.56\small{$\pm$0.81} & 86.79\small{$\pm$0.30} & 73.54 & 89.32 & 83.85\\

DomainAdaptor (CVPR'23)~\cite{zhang2023domainadaptor} & 77.23\small{$\pm$3.97} & 90.22\small{$\pm$0.40} & 84.20\small{$\pm$0.11} & 76.41\small{$\pm$4.28} & 91.80\small{$\pm$0.31} & 85.73\small{$\pm$0.12} & 70.17\small{$\pm$8.01} & 91.32\small{$\pm$0.24} & 83.50\small{$\pm$0.09} & 67.39\small{$\pm$9.82} & 84.16\small{$\pm$0.21} & 78.02\small{$\pm$0.08} & 76.97\small{$\pm$4.59} & 89.30\small{$\pm$0.13} & 85.46\small{$\pm$0.08} & 73.63 & 89.36 & 83.38 \\

DeY-Net (WACV'24)~\cite{wen2024denoising} 
 & 80.03\small{$\pm$8.31} &94.42\small{$\pm$0.20} & 86.35\small{$\pm$0.84} & {84.30\small{$\pm$7.09}} & \textcolor{red}{94.25\small{$\pm$0.23}} & \textcolor{red}{87.16\small{$\pm$0.47}} & 80.32\small{$\pm$7.85} & 93.40\small{$\pm$0.32} & 88.41\small{$\pm$0.33} & 78.67\small{$\pm$5.31} & 86.12\small{$\pm$0.78} & 80.45\small{$\pm$0.30} & 76.81\small{$\pm$3.79} &90.09\small{$\pm$0.55} & 86.30\small{$\pm$0.25} & 80.02 & 91.65&85.73\\

VPTTA (CVPR'24)~\cite{chen2024each} &73.57\small{$\pm$6.60} & 92.68\small{$\pm$0.03} & 84.14\small{$\pm$0.01} & 78.21\small{$\pm$2.40} & 94.07\small{$\pm$0.09} & 86.16\small{$\pm$0.01} & 69.26\small{$\pm$4.29} & 92.78\small{$\pm$0.08} & 82.66\small{$\pm$0.02} & 60.11\small{$\pm$8.05} & 85.18\small{$\pm$0.10} & 76.24\small{$\pm$0.03} & 72.58\small{$\pm$5.21} & 91.16\small{$\pm$0.13} & 84.74\small{$\pm$0.04} & 70.74 & 91.17 & 82.78\\ 

NC-TTT (CVPR'24)~\cite{osowiechi2024nc} & 78.21\small{$\pm$2.74} & 93.87\small{$\pm$0.25} & 85.49\small{$\pm$0.11} & 82.13\small{$\pm$3.30} & 93.19\small{$\pm$0.29} & 86.78\small{$\pm$0.08} & 77.50\small{$\pm$5.29} & 91.99\small{$\pm$0.14} & 84.08\small{$\pm$0.03} & 74.14\small{$\pm$3.50} & 87.53\small{$\pm$0.25} & 80.56\small{$\pm$0.11} & 80.53\small{$\pm$1.08} & 92.73\small{$\pm$0.10} & 85.81\small{$\pm$0.07} & 78.50 & 91.86 & 84.54
  \\  

\hline

Ours & \textcolor{red}{85.25\small{$\pm$2.33}} &  \textcolor{red}{94.68\small{$\pm$0.09}} & \textcolor{red}{88.52\small{$\pm$0.13}} & \textcolor{red}{85.34\small{$\pm$3.08}} & 93.18\small{$\pm$0.20} & 86.83\small{$\pm$0.11} & \textcolor{red}{86.19\small{$\pm$1.99}} & \textcolor{red}{94.57\small{$\pm$0.20}} & \textcolor{red}{89.24\small{$\pm$0.13}} & \textcolor{red}{81.52\small{$\pm$4.25}} & \textcolor{red}{91.20\small{$\pm$0.30}} & \textcolor{red}{84.53\small{$\pm$0.28}} & \textcolor{red}{86.08\small{$\pm$3.08}} & \textcolor{red}{94.60\small{$\pm$0.23}} & \textcolor{red}{88.58\small{$\pm$0.11}} & \textcolor{red}{84.87} & \textcolor{red}{93.64} & \textcolor{red}{87.54} \\ 

\hlineB{3}
\end{tabular}

  \end{adjustbox}
  \label{tab:tabel_single_fundus}
\end{table*}

For the retinal fundus segmentation task, we conducted single-source domain generalization experiments. Unlike the experiments described in the main text, this setup simulates a more realistic scenario where test data may originate from arbitrarily complex real-world distributions, i.e., mixed distribution shifts. Specifically, data from one site was selected and split $8:2$ into training and validation sets ($S=1$), while the remaining sites ($T=|\mathcal{D}_s \cup \mathcal{D}_t|-1$) were shuffled and used entirely as the testing dataset. Notably, all models encountered these target domains for the first time during testing.

As shown in Table~\ref{tab:tabel_single_fundus}, our approach achieved SOTA performance across all five transfer experiments for the DSC metric. In the average results, we outperformed the second-best method (DeY-Net~\cite{wen2024denoising}) by 4.85\%, 1.78\%, and 1.81\% in the DSC, $E_\phi^{max}$, and $S_{\alpha}$, respectively. These results validate the effectiveness of our method for medical image segmentation tasks.
        
The retinal fundus dataset is characterized by significant low-level visual differences and features segmentation targets that are not singular, often exhibiting overlapping and fixed structures. We attribute our superior performance to the comprehensive learning of the morphological knowledge of organs. This enables our method to robustly distinguish organ instances and their shape features—a domain-invariant property—even under severe domain shifts that degrade the performance of other methods.

\subsection{Multi Source DG in MRI Prostate}
\noindent \textbf{Datasets.}
We conducted experiments on the prostate segmentation task using T2-weighted MRI scans collected from six different clinical centers, denoted as Domain RUNMC, BMC, I2CVB, UCL, BIDMC, and HK. These centers are sourced from three publicly available datasets: NCI-ISBI13~\cite{nicholas2015nci}, I2CVB~\cite{lemaitre2015computer}, and PROMISE12~\cite{litjens2014evaluation}.

\noindent \textbf{Implementation Details.}
We followed the data preprocessing pipeline of \cite{zhang2024pass} to ensure consistency. Specifically, we used 30 labeled cases from RUNMC as the source dataset and evaluated the model on 30, 19, 13, 12, and 12 unlabeled cases from the five remaining clinical sites. Each MRI axial slice was resized to $384 \times 384$ pixels and normalized to have zero mean and unit variance. Before normalization, we clipped the 5\%–95\% intensity range of the histograms to reduce outlier influence.

For feature extraction, we employed a ResNet-50 backbone pre-trained on ImageNet. During both the source model training and test-time adaptation (TTA) stages, we maintained a batch size of 8.
Given that edge precision is crucial in MRI prostate segmentation, and considering the complex shape variations of the prostate, we selected Dice Score (DSC) and Hausdorff Distance (HD95) as the primary evaluation metrics to provide a comprehensive performance assessment.

\noindent \textbf{Experimental Results.}
\begin{table*}[h!]
\centering
  \caption{Test-time domain generalization results on the MRI prostate datasets.  The performance (mean $\pm$ standard deviation) of three trials for our method and six SOTA methods. Best results are colored as \textcolor{red}{red}.}
  \begin{adjustbox}{width=0.969\linewidth}
    \begin{tabular}{c|cccccccccc|cc}
\hlineB{3}
\multirow{2}{*}{Methods} & \multicolumn{2}{c}{BMC} & \multicolumn{2}{c}{I2CVB} & \multicolumn{2}{c}{UCL} & \multicolumn{2}{c}{BIDMC}  & \multicolumn{2}{c|}{HK} & \multicolumn{2}{c}{Avg.}   \\ \cline{2-13} 
                         & \textit{DSC}      &  HD95    & \textit{DSC}      &  HD95& \textit{DSC}      &  HD95& \textit{DSC}     &  HD95  & \textit{DSC}      &  HD95  & \textit{DSC} $\uparrow$     &  HD95 $\downarrow$ \\ \hline \hline

\textit{No Adapt} & 74.30\small{$\pm$5.31} & 16.08\small{$\pm$12.41} & 66.47\small{$\pm$13.50} & 37.16\small{$\pm$18.24} & 75.28\small{$\pm$6.20} & 16.77\small{$\pm$12.10} & 52.08\small{$\pm$7.71} & 50.09\small{$\pm$20.85} & 80.51\small{$\pm$9.35} & 8.79\small{$\pm$9.06} & 69.72\small{$\pm$8.29} & 25.77\small{$\pm$15.41}  \\  \hline

TENT (ICLR'21)~\cite{wangtent} & 77.45\small{$\pm$3.79} & 12.09\small{$\pm$9.88} & 69.10\small{$\pm$10.47} & 30.78\small{$\pm$19.22} & 79.69\small{$\pm$4.81} & 14.71\small{$\pm$11.01} & 52.01\small{$\pm$6.80} & 42.63\small{$\pm$10.13} & 84.58\small{$\pm$2.73} & 4.07\small{$\pm$5.38} & 72.56\small{$\pm$4.26} & 20.85\small{$\pm$13.64} \\

TASD (AAAI'22)~\cite{liu2022single} & 76.28\small{$\pm$2.35} & 15.11\small{$\pm$15.17} & 68.30\small{$\pm$7.88} & 31.43\small{$\pm$24.10} & 80.25\small{$\pm$3.54} & \textcolor{red}{10.59\small{$\pm$16.39}} & 56.08\small{$\pm$3.82} & 51.90\small{$\pm$24.82} & 81.09\small{$\pm$1.79} & 4.26\small{$\pm$4.16} & 72.40\small{$\pm$5.72} & 22.65\small{$\pm$17.53}  \\

SAR (ICLR'23)~\cite{niu2023towards} & 77.24\small{$\pm$4.26} & 20.48\small{$\pm$10.12} & 68.99\small{$\pm$8.27} & 49.07\small{$\pm$15.66} & 79.27\small{$\pm$8.48} & 18.03\small{$\pm$5.89} & 50.81\small{$\pm$10.60} & 54.35\small{$\pm$19.31} & 85.40\small{$\pm$3.08} & 3.87\small{$\pm$3.55} & 72.34\small{$\pm$4.80} & 29.16\small{$\pm$16.28} \\

DomainAdaptor (CVPR'23)~\cite{zhang2023domainadaptor} & 76.49\small{$\pm$2.59} & 19.27\small{$\pm$8.13} & 69.07\small{$\pm$9.14} & 32.57\small{$\pm$10.40} & 80.41\small{$\pm$5.08} & 16.24\small{$\pm$9.88} & 49.99\small{$\pm$14.28} & 48.40\small{$\pm$10.28} & 85.20\small{$\pm$1.90} & 3.25\small{$\pm$6.94} & 72.23\small{$\pm$6.82} & 23.94\small{$\pm$14.55} \\

VPTTA (CVPR'24)~\cite{chen2024each}& 77.42\small{$\pm$4.38} & 12.93\small{$\pm$7.09} & 70.25\small{$\pm$5.18} & 30.01\small{$\pm$13.68} 
 & 82.07\small{$\pm$6.27} & 13.28\small{$\pm$18.09} & 57.49\small{$\pm$8.46} & 40.11\small{$\pm$12.05} & 83.27\small{$\pm$2.96} & 3.40\small{$\pm$5.45} & 74.10\small{$\pm$4.79} & 19.94\small{$\pm$12.99}   \\

PASS (TMI'24)~\cite{zhang2024pass} & \textcolor{red}{80.07\small{$\pm$7.14}} & 10.50\small{$\pm$9.57} &71.41\small{$\pm$6.28} & 28.26\small{$\pm$9.97} & 84.39\small{$\pm$8.81} & 10.68\small{$\pm$12.27} & 57.27\small{$\pm$11.48} & \textcolor{red}{36.94\small{$\pm$16.43}} & 84.88\small{$\pm$3.71} & 3.03\small{$\pm$5.05} & 75.60\small{$\pm$5.13} & 17.88\small{$\pm$12.14} \\ \hline

Ours & 79.63\small{$\pm$4.71} & \textcolor{red}{9.99\small{$\pm$11.10}} & \textcolor{red}{74.09\small{$\pm$9.80}} & \textcolor{red}{25.70\small{$\pm$13.07}} & \textcolor{red}{86.30\small{$\pm$7.25}} & 11.08\small{$\pm$17.47} & \textcolor{red}{60.33\small{$\pm$13.59}} & 39.52\small{$\pm$14.20}  & \textcolor{red}{86.12\small{$\pm$2.08}} & \textcolor{red}{2.84\small{$\pm$8.08}} & \textcolor{red}{77.29\small{$\pm$3.98}} & \textcolor{red}{17.82\small{$\pm$11.06}}  \\

\hlineB{3}
\end{tabular}

  \end{adjustbox}
  \label{tab:tabel_prostate}
\end{table*}
The MRI prostate segmentation results are presented in Table~\ref{tab:tabel_prostate}, where we compare our method against several SOTA approaches, including the latest TTA segmentation method, PASS~\cite{zhang2024pass}.
As shown in the results, the performance of existing TTA methods remains relatively close across both DSC and HD95 metrics. While PASS exhibits strong segmentation performance, our method surpasses it with a 1.69\% improvement in DSC, demonstrating its effectiveness.
Given the inherent challenges of MRI prostate segmentation, characterized by diverse imaging modalities and complex morphological variations, our results highlight the robust generalization capability of our approach. Nonetheless, further enhancing edge precision remains an important focus for our future work.

\subsection{Natural Image Classification}
\noindent \textbf{Datasets.} For natural image classification tasks, we selected two benchmark datasets: PACS~\cite{li2017deeper} and VLCS~\cite{fang2013unbiased}, which are widely used in domain generalization and test-time adaptation studies. 
The PACS~\cite{li2017deeper} dataset consists of large images spanning 7 classes evenly distributed across 4 domains, i.e. A (Art), C (Cartoons), P (Photos), and S (Sketches), with a total of 9,991 images. 
The VLCS~\cite{fang2013unbiased} dataset comprises 10,729 images across 5 classes (bird, car, chair, dog, and person), evenly distributed across 4 domains: C (Caltech101), L (LabelMe), S (SUN09), and V (VOC2007).

\noindent \textbf{Source model training.} For all experiments, we employed an ImageNet-pretrained ResNet-50~\cite{he2016deep} as the feature extractor, with an MLP layer provided by the DomainBed~\cite{gulrajanisearch} benchmark serving as the classifier. We used the SGD optimizer with a learning rate of $1 \times 10^{-5}$. The batch size was set to 32, and training was conducted for 10,000 iterations. All images were resized to $224 \times 224$, and data augmentation techniques—including random cropping, flipping, color jittering, and intensity adjustments—were applied during source training.

\noindent \textbf{Implementation details of test-time adaptation setup.} We evaluated our framework against six methods (i.e. Empirical Risk Minimization (ERM)~\cite{vapnik1998statistical}, DomainAdaptor~\cite{zhang2023domainadaptor}, ITTA~\cite{chen2023improved}, VPTTA~\cite{chen2024each}, NC-TTT~\cite{osowiechi2024nc}) under fair comparison conditions, following the leave-one-out training strategy using the publicly available DomainBed~\cite{gulrajanisearch} framework. For deploying our framework at the test-time phase, we employed SGD with a learning rate of 0.005, a batch size of 16, and a universe size of $d = 60$. Notably, as all images in the natural image classification task contain only a single instance class, the class-wise similarity matrix described in Section 3.2 of the main text was not utilized.

\noindent \textbf{Experimental Results.} The classification results across different domains for natural images are presented in Tables~\ref{tab:tabel_natural_pacs} and \ref{tab:tabel_natural_vlcs}. While the ERM method shows strong performance compared to existing approaches, our method achieves higher classification accuracy, surpassing ERM by 3.01\% on the PACS dataset and 2.49\% on the VLCS dataset. Additionally, our approach demonstrates competitive performance against state-of-the-art test-time adaptation methods designed for natural images. Natural images present greater challenges compared to medical images due to the lack of consistent morphological priors typically observed in the latter. However, unlike segmentation tasks, classification does not require determining the specific class of every pixel in an image. Our framework’s graph construction effectively captures spatial correspondences for each instance, further enhancing its performance.

\begin{table}[h!]
\centering
  \caption{Test-time domain generalization results on the PACS~\cite{li2017deeper} dataset using a ResNet-50 backbone. Each column (A, C, P, S) indicates the domain used as the test set, while the remaining domains are used for training. The best results are highlighted in \textcolor{red}{red}.}
  \begin{adjustbox}{width=0.969\linewidth}
    \begin{tabular}{c|ccccc}
    \hlineB{3}
        Method & A & C & P & S & Avg. \\ \hline
        ERM~\cite{vapnik1998statistical} & 84.07 & 80.21 & 97.06 & 81.99 & 85.83 \\ \hline
        TENT (ICLR'21)~\cite{wangtent} & 82.34 & 78.63 & 97.93 & 82.72 & 85.40 \\ 
        DomainAdaptor (CVPR'23)~\cite{zhang2023domainadaptor} & 86.15 & 82.02 & 98.40 & 84.38 & 88.45 \\ 
        ITTA (CVPR'23)~\cite{chen2023improved}& 85.63 & \textcolor{red}{84.30} & 97.27 & 84.09 & 87.82 \\
        VPTTA (CVPR'24)~\cite{chen2024each} & \textcolor{red}{86.50}  & 83.77 & 97.09 & 85.10 & 88.12 \\ 
        NC-TTT (CVPR'24)~\cite{osowiechi2024nc}& 83.81 & 80.44 & 96.53 & 82.36 & 85.79 \\ 
        Ours & 85.08 & 83.93 & \textcolor{red}{98.61} & \textcolor{red}{87.76} & \textcolor{red}{88.84} \\ \hlineB{3}
\end{tabular}
  \end{adjustbox}
  \label{tab:tabel_natural_pacs}
\end{table}

\begin{table}[h!]
\centering
  \caption{Test-time domain generalization results on the VLCS~\cite{fang2013unbiased} dataset using a ResNet-50 backbone. Each column (C, L, S, V) indicates the domain used as the test set, while the remaining domains are used for training. The best results are highlighted in \textcolor{red}{red}.}
  \begin{adjustbox}{width=0.969\linewidth}
    \begin{tabular}{c|ccccc}
    \hlineB{3}
        Method & C & L & S & V & Avg. \\ \hline
        ERM~\cite{vapnik1998statistical} & 97.63 & 64.20 & 70.39 & 74.41 & 76.66 \\ \hline
        TENT (ICLR'21)~\cite{wangtent} & 96.88 & 64.46 & 71.07 & 73.52 & 76.48 \\ 
        DomainAdaptor (CVPR'23)~\cite{zhang2023domainadaptor} & \textcolor{red}{98.69} & \textcolor{red}{69.18} & 73.66 & 76.01 & \textcolor{red}{79.39} \\ 
        ITTA (CVPR'23)~\cite{chen2023improved}& 97.30 & 66.09 & 72.31 & 75.10 & 77.70 \\
        VPTTA (CVPR'24)~\cite{chen2024each} & 97.25  & 67.69 & 71.78 & 75.22 & 77.98 \\ 
        NC-TTT (CVPR'24)~\cite{osowiechi2024nc}& 96.72 & 65.58 & 73.04 & \textcolor{red}{76.83} & 78.04 \\ 
        Ours & 98.49 & 68.72 & \textcolor{red}{74.12} & 75.30 & 79.15 \\ \hlineB{3}
\end{tabular}
  \end{adjustbox}
  \label{tab:tabel_natural_vlcs}
\end{table}

\subsection{Additional Visualization}
We conducted additional visualization experiments, with segmentation results for retinal fundus and polyp images shown in Figures \ref{fig:vis_appendix_fundus} and \ref{fig:vis_appendix_poloy}, respectively. Each row represents images from a distinct domain (Site), and we ensured that the model performing inference had not encountered images from that domain before. 

\begin{figure*}[!t]
    \centering
    \includegraphics[width=0.999\linewidth]{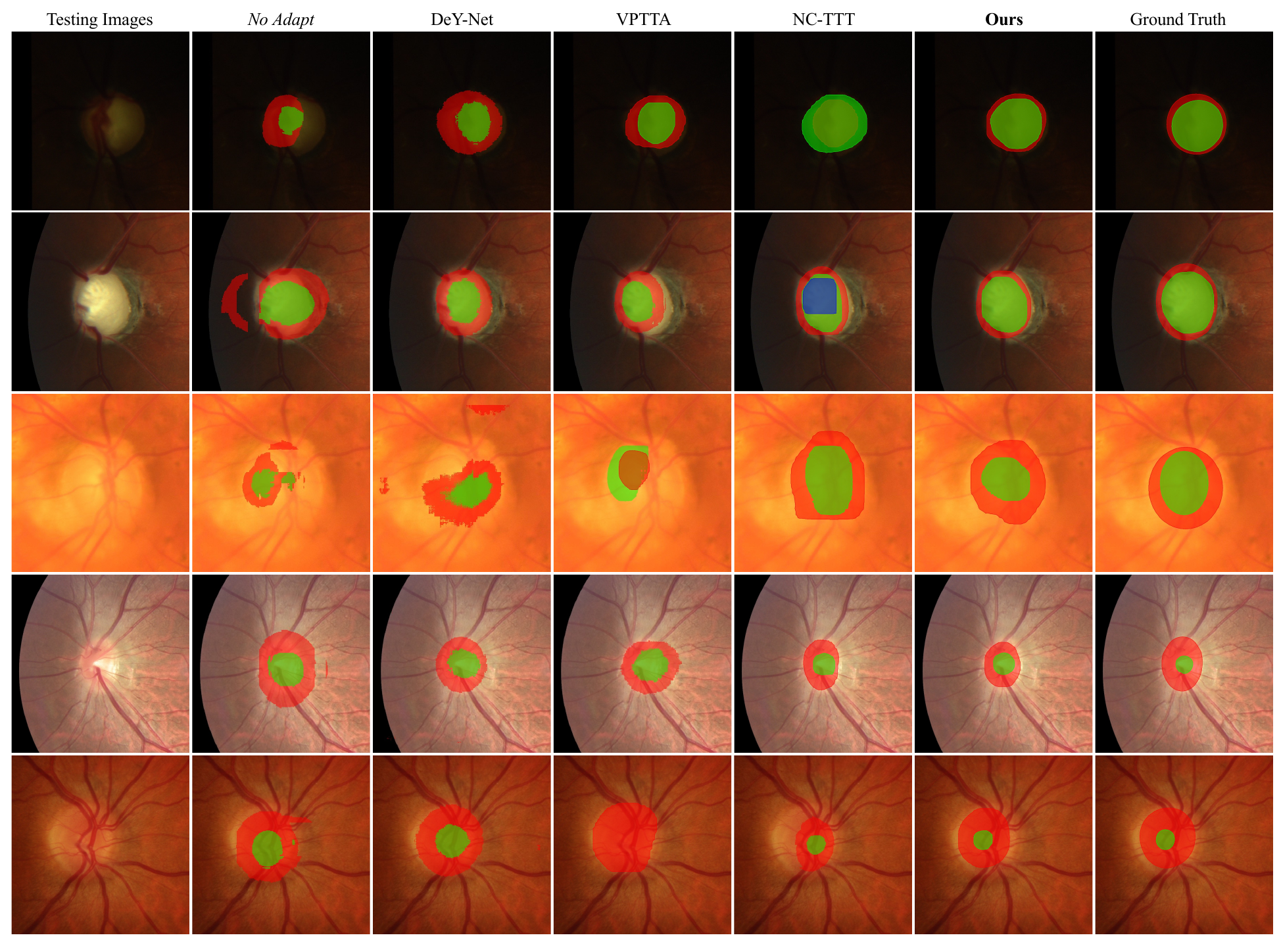}
    \caption{Visualization comparison of segmentation results for the \textit{No Adapt} baseline, DeY-Net~\cite{wen2024denoising}, VPTTA~\cite{chen2024each}, NC-TTT~\cite{osowiechi2024nc}, and our method in retinal fundus segmentation. The five rows from top to bottom display the final segmentation results for tests conducted on Sites A to E. Different colors represent the segmentation instances of different classes identified by the network.}
    \label{fig:vis_appendix_fundus}
\end{figure*}

\begin{figure*}[!t]
    \centering
    \includegraphics[width=0.999\linewidth]{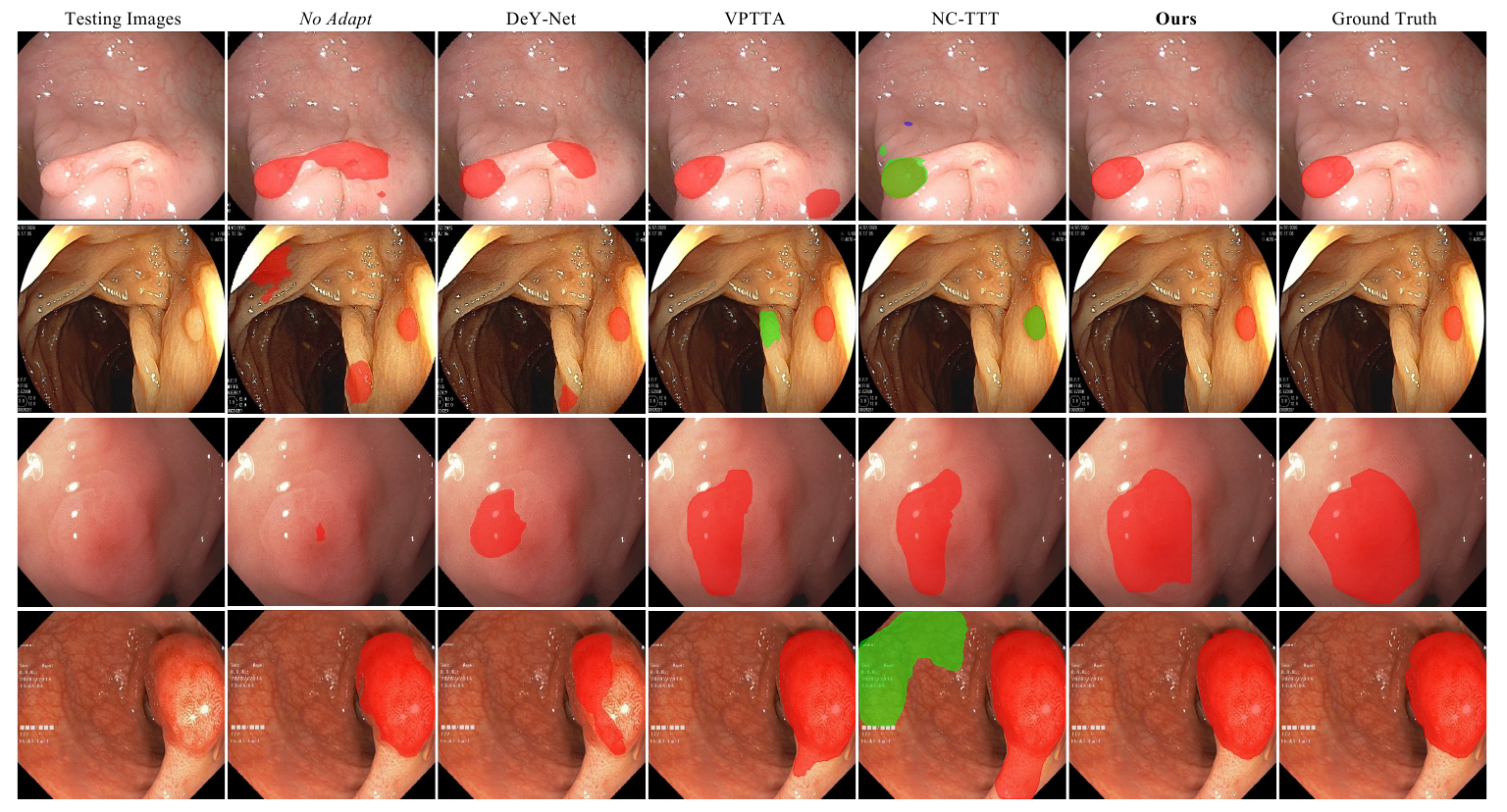}
    \caption{Visualization comparison of segmentation results for the \textit{No Adapt} baseline, DeY-Net~\cite{wen2024denoising}, VPTTA~\cite{chen2024each}, NC-TTT~\cite{osowiechi2024nc}, and our method in polyp segmentation. The four rows from top to bottom display the final segmentation results for tests conducted on Sites A to D. Different colors represent the segmentation instances of different classes identified by the network.}
    \label{fig:vis_appendix_poloy}
\end{figure*}

Retinal fundus segmentation, in particular, presents a challenge due to the presence of two overlapping substructures. Lower clarity and contrast in images (e.g., rows 1 and 2 of Figure~\ref{fig:vis_appendix_fundus}) further complicate the model’s ability to accurately differentiate and segment these structures. By incorporating morphological priors of the organ within a multi-graph matching network, our method effectively learns robust substructure representations while minimizing domain-related noise. This approach overcomes issues like repeated, missing, or blurred edge pixels commonly seen in other methods, providing a more precise segmentation outcome.

The segmentation of polyps presents a greater challenge than that of retinal fundus imaging due to the highly variable appearance, with marked differences in shape, size, and color across domains. This variability demands precise, pixel-level classification from the network. Furthermore, we have not designated polyp segmentation as a single-object task; the model independently classifies and segments multiple classes during testing, using different colors to distinguish each segmented object in the visualization. As illustrated in Figure~\ref{fig:vis_appendix_poloy}, the masks generated by our method are in close alignment with expert annotations and effectively avoid pixel misclassification into different categories, a common issue in other methods.

\section{Additional Analysis}
\subsection{Effectiveness of the class-wise similarity matrix}
The class-wise similarity matrix $\bold{W}$ is introduced to mitigate category confusion in graphs caused by nodes belonging to different classes. Such confusion often results in mismatches, semantic deviations, and redundant computations. By reordering the adjacency matrix based on the labels $Y_i$ of each node $\mathcal{V}_i$, our method strengthens the capacity to identify and learn class-specific information during the source training phase.
To validate the above perspective, we conducted experiments comparing the final TTA segmentation results with and without $\bold{W}$ (denoted as with $\bold{W}$ and \textit{w/o} $\bold{W}$). As illustrated in Figure~\ref{fig:class_simi_ma}, \textit{w/o} $\bold{W}$ results in a measurable decline in DSC performance. Furthermore, we visualized the effect of \textit{w/o} $\bold{W}$ in multi-object segmentation scenarios, as shown in Figure~\ref{fig:class_w}. While the masks generated by the model closely align with the ground truth, the model misclassified the categories of two segmented instances.

\subsection{Effectiveness of Morphological Priors}
\begin{figure}[h]
    \centering
    \includegraphics[width=0.999\linewidth]{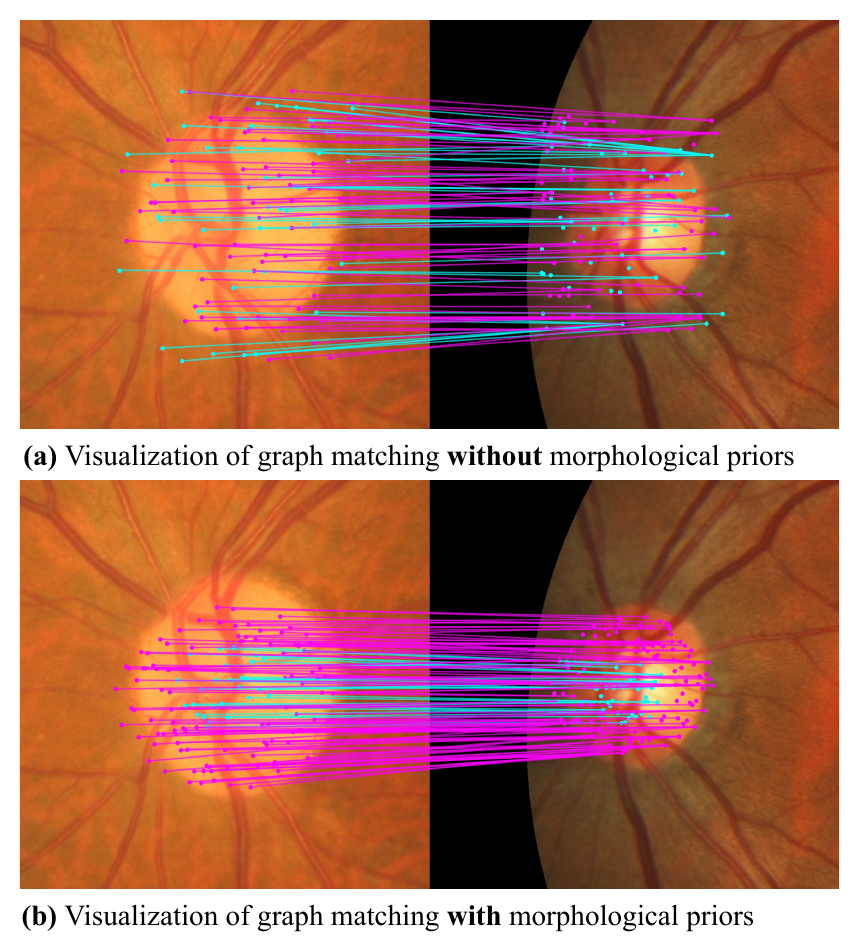}
    \caption{Visualization of graph pair matching.}
    \label{fig:re_visual}
\end{figure}
We visualized cross-site pairing without morphological priors, as shown in Fig. ~\ref{fig:re_visual}(a), and compared it with the results obtained after incorporating priors, as shown in Fig. ~\ref{fig:re_visual}(b). Without priors, the graph nodes were not correctly sampled within the corresponding organs, leading to mismatches. 
By introducing priors, this issue was effectively resolved, and multigraph matching ensured more stable pairing across multiple domains. For a quantitative evaluation of the impact of without priors, please refer to Table 4 in the main text.

\begin{figure}[!t]
    \centering
    \includegraphics[width=0.899\linewidth]{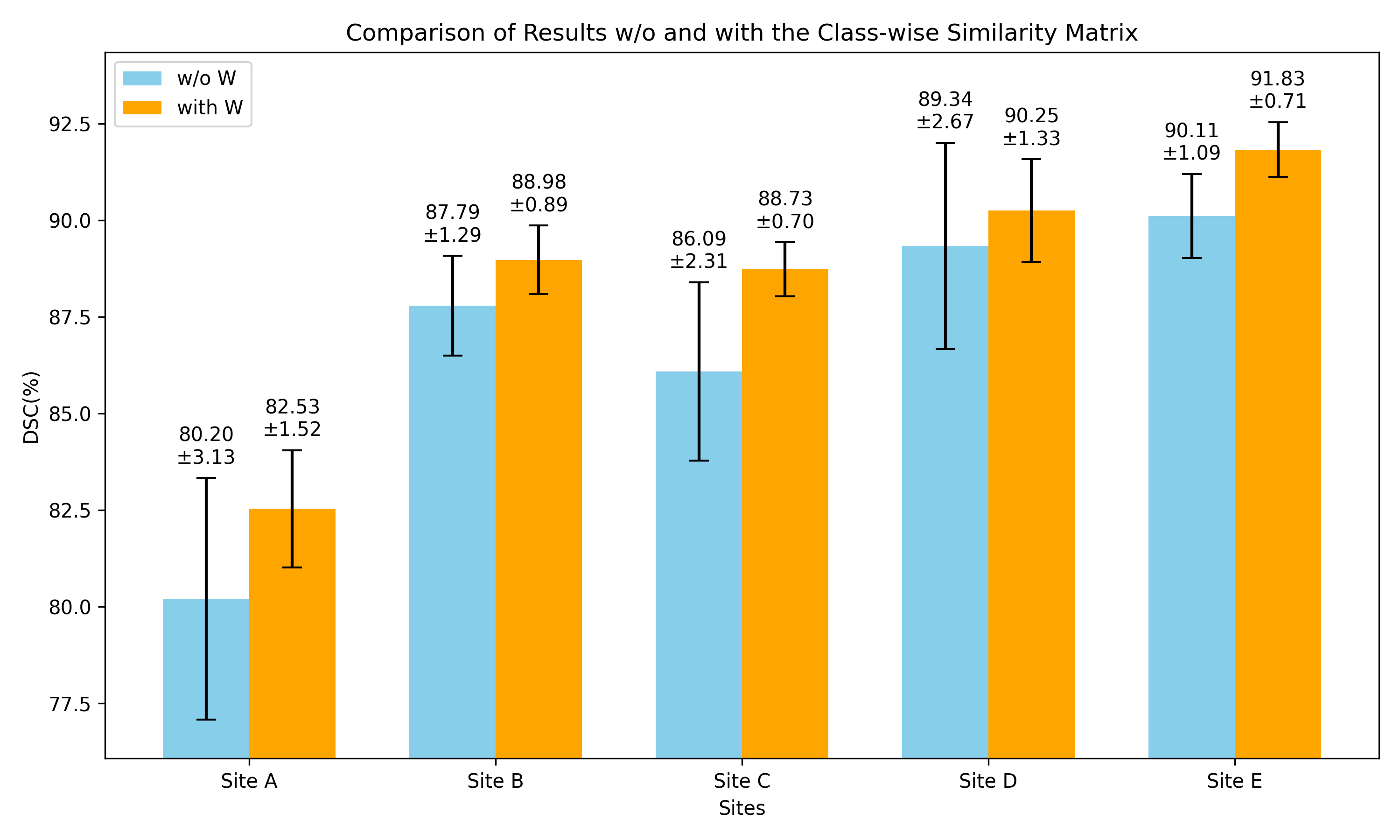}
    \caption{Ablation study on the impact of the Class-wise Similarity Matrix $\bold{W}$ in retinal fundus segmentation: comparison of results with and without (\textit{w/o}) $\bold{W}$.}
    \label{fig:class_simi_ma}
\end{figure}

\begin{figure}[!t]
    \centering
    \includegraphics[width=0.999\linewidth]{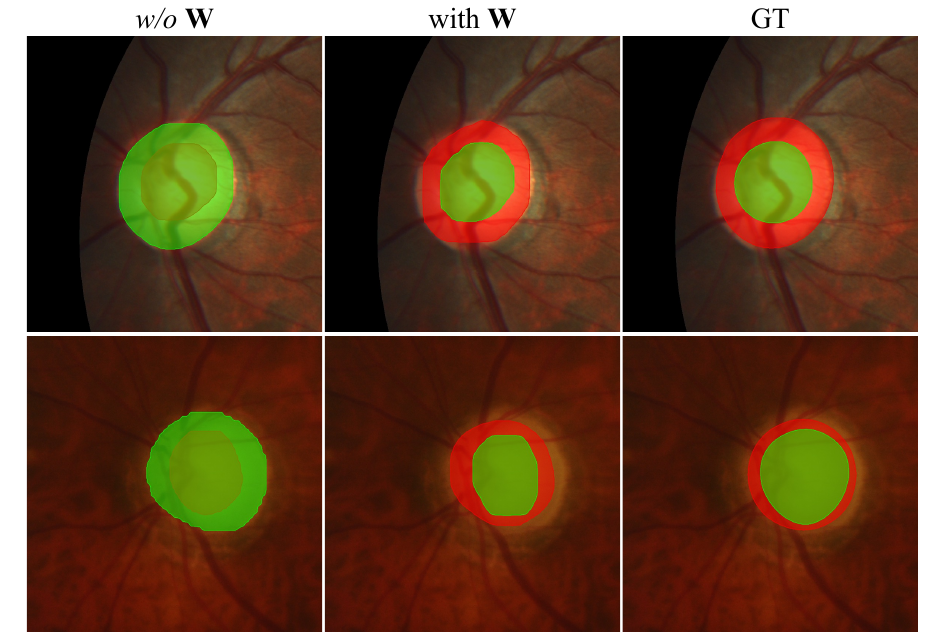}
    \caption{Visualization of segmentation results with and without (\textit{w/o}) the Class-wise Similarity Matrix $\bold{W}$.}
    \label{fig:class_w}
\end{figure}

\subsection{Impact of Batch Size on Segmentation}
\begin{table}[h!]
\centering
  \caption{Ablation study on batch size during TTA for retinal fundus segmentation. ``Avg. DSC" represents the average DSC across the five sites, while ``time" indicates the inference time per image.}
  \begin{adjustbox}{width=0.669\linewidth}
    \begin{tabular}{c|c|c|c}
\hlineB{3}
Batch Size & Avg. DSC & FLOPs (G) & time (s/img) \\ \hline
2 & 85.20 & 3.012 & 0.277 \\
4 & 88.46 & 4.255 & 0.392 \\
8 & 88.93 & 20.43 & 0.780 \\
16 & 89.15 & 80.96 & 1.831 \\
32 & 88.31 & 223.1 & 3.715 \\

\hlineB{3}
\end{tabular}
  \end{adjustbox}
  \label{tab:batchsize}
\end{table}

As shown in Table~\ref{tab:batchsize}, increasing the number of simultaneously matched graphs leads to a significant increase in both FLOPs and inference time, while the improvement in segmentation quality remains marginal. To achieve a balance between segmentation performance and computational efficiency, we set the mini-batch size to 4 during the TTA phase in retinal fundus datasets. However, this is a tunable hyperparameter rather than a fixed value, as it depends on factors such as the size of the segmented objects and the input image resolution. Empirically, we find that a mini-batch size between 4 and 8 provides an optimal trade-off.

\section{Limitations}
Unlike mainstream TTA methods that update only the Batch Normalization layers, our approach optimizes all network parameters during test time, achieving superior segmentation performance. However, the increased computational overhead limits deployment on portable devices, making efficiency optimization a key focus for future work.

In our experiments, we also observed that when both large and small organs are present, the model tends to perform better on larger organs while often overlooking smaller ones. This is due to the uniform sampling of foreground nodes, which can lead to diminished segmentation accuracy for small targets. To address this, we plan to incorporate stronger regularization in future work to better guide the sampling and learning of small structures.

Our method is well-suited for medical imaging compared to natural image tasks. In natural images, objects often exhibit significant variation due to intrinsic properties, motion, and state changes. 

\begin{algorithm*}
    \caption{\textbf{Source Training Phase per Mini-Batch}}
    \renewcommand{\algorithmicrequire}{\textbf{Input:}}     \renewcommand{\algorithmicensure}{\textbf{Output:}}
    
    \begin{algorithmic}[1]
        \ENSURE 
        $\mathcal{L}_{overall}$: The overall loss for training the segmentation network;

        $\mathcal{U}$: the pre-trained universe embeddings integrate morphological priors;

        \REQUIRE 
        $\{ x_i \in \mathbb{R}^{H\times W\times C} \}_{i=1}^m$: A batch of $m$ images from one or multiple domains;

        $\{ y_i \in [0,255] \cap \mathbb{Z} \}_{i=1}^m$: The ground truth masks corresponding to the input images;

        $E(\cdot)$: Feature extractor (ResNet-50);

        $S(\cdot)$: Segmentation head;

        $N$: The total classes number of segmentation organ;

        $\mathcal{U}$: Learnable universe embeddings;
        
        \textit{\textbf{(1) Segmentation Network Training.}}
        \STATE Get the visual feature maps: $f_i \leftarrow E(x_i)$.

        \STATE Get the predict segmentation masks: $\hat{y}_i \leftarrow S(f_i)$.

        \STATE Get the supervised loss: $\mathcal{L}_{sup} \leftarrow \text{CE}(\hat{y}_i, y_i)$, where CE is Cross Entropy Loss.

        \textit{\textbf{(2) Graph Construction.}}
        \FOR{each $i \in [1,m]$}
            \FOR{each object $n \in [1, N]$} 
                \STATE $\{f_{i,k}^n\}_{k=1}^{K} \leftarrow$ Extract feature maps for object $n$ from layers $1$ to $K$ based on $f_i$ and $y_i$.
            \ENDFOR
        
            \STATE Obtain object-specific features: $\{F_{i}^n\}_{n=1}^N \leftarrow Concat(f_{i,1}^n, \dots, f_{i,K}^n)$ for each $n$ in $N$.
        
            \STATE Build features of nodes and corresponding labels: $\{\mathcal{V}_i \in \mathbb{R}^{n_i \times h}, {Y}_i \in \mathbb{Z}^{n_i}\}_{i=1}^m \leftarrow \phi(\{F_i^n\}_{n=1}^N)$, where $\phi$ is the spatially-uniform sampling, and $n_i$ is the total number of nodes for $x_i$. 

            \STATE $ \mathcal{G}_i = (\mathcal{V}_i, \mathcal{A}_i)$, the weighted adjacency matrix $\mathcal{A}_i$ is obtained from Eq. (5).
        \ENDFOR

        \textit{\textbf{(3) Formulation of \textit{universe embeddings}.}}

        \IF{ $\mathcal{U}$ is not initialized}
        \STATE $\mathcal{U} = 1/d + 10^{-3}z$, where $z \sim N(0, 1)$.
        \ENDIF
        \STATE \textit{Universe matching matrices}: $\bold{U} = [U_1^{\mathsf{T}}, \cdots, U_m^{\mathsf{T}}]^{\mathsf{T}}$, where $U_i = Sinkhorn(\mathcal{V}_i \ \mathcal{U}^{\mathsf{T}}, \tau) \in \mathbb{U}_{n_i d}$, $d$ is the \textit{universe size}.

        \STATE Block-diagonal multi-adjacency matrix: $\bold{A} = diag(\mathcal{A}_1, \cdots, \mathcal{A}_m)$. 

        \STATE Compute the class-aware similarity matrix: $\bold{\tilde{A}} = \bold{W}^{\mathsf{T}} \bold{A}\bold{W}$, where $\bold{W} = [W_{ij}]_{ij}$, and $W_{ij} = {Y}_i {Y}_j^{\mathsf{T}}$. 

        \STATE HiPPI solving for stable convergence of $\bold{U}$ as in Eqs. (6-8).

        \STATE Update $\mathcal{U}$ with $L(\mathcal{U})$ in Eq. (9).

        \textit{\textbf{Overall Loss of Source Training.}}
        \STATE $\mathcal{L}_{overall}=\mathcal{L}_{sup}+L(\mathcal{U})$.

    \end{algorithmic}
    \label{algo:train}
\end{algorithm*}

\begin{algorithm*}
\caption{\textbf{Higher-order Projected Power Iteration (HiPPI)~\cite{bernard2019hippi}}}
    \renewcommand{\algorithmicrequire}{\textbf{Input:}}     \renewcommand{\algorithmicensure}{\textbf{Output:}}
    
    \begin{algorithmic}[1]
        \ENSURE 
        Cycle-consistent universe-matching $\bold{U}_t$.

        \REQUIRE 
        $W$: multi-graph similarity matrix;

        $\bold{U}_0$: initial universe-matching $\bold{U}_0 \in \mathbb{U}_{n_id}$;

        \STATE \textbf{Initialise:} $t\leftarrow 0$, $\text{proj} \leftarrow Sinkhorn$.
        \REPEAT
        \STATE \hspace{1em} $V_t \leftarrow W \bold{U}_t \bold{U}_t^{\mathsf{T}} W \bold{U}_t$.
        \STATE \hspace{1em} $\bold{U}_{t+1} \leftarrow \text{proj}(V_t)$.
        \STATE \hspace{1em} $t\leftarrow t+1$.
        \UNTIL{ $|| U^t_i - U^{t-1}_i|| < 10^{-5}$}
    \end{algorithmic}
\label{algo:hippi}
\end{algorithm*}

\begin{algorithm*}
    \caption{\textbf{Test-time Adaptation Phase per Mini-Batch}}
    \renewcommand{\algorithmicrequire}{\textbf{Input:}}     \renewcommand{\algorithmicensure}{\textbf{Output:}}
    
    \begin{algorithmic}[1]
        \ENSURE 
        $\{y^{*}_i \in \mathbb{R}^{H\times W\times C} \}_{i=1}^m$: The predicted masks of test dataset.

        \REQUIRE 
        $\{ x_i^t \in \mathbb{R}^{H\times W\times C} \}_{i=1}^m$: A batch of $m$ images from test dataset;


        $E(\cdot)$: Pre-trained feature extractor (ResNet-50);

        $S(\cdot)$: Pre-trained segmentation head;

        $N$: The total classes number of segmentation organ;

        $\mathcal{U}$: Pre-trained learnable universe embeddings;

        \textit{MIter}: Max iteration of multi-graph matching.

        \STATE \textbf{Initialise:} \textit{Iter}$\leftarrow 0$. ${V}_i \leftarrow 0$, $\forall i \in [m]$. ${V}_i$ is the gradient of Eq. (4) with respect to ${U}_i$. $\tau \leftarrow 0.05$.

        \STATE Obtain the visual feature maps: $f_i^t \leftarrow E(x_i^t).$

        \STATE Obtain the pseudo segmentation masks: $\hat{y}^t_i \leftarrow S(f_i^t).$

        \textit{\textbf{(1) Graph Construction.}}

        \STATE $ \mathcal{G}_i^t = (\mathcal{V}_i^t, \mathcal{A}_i^t)$, where $\mathcal{V}_i^t$ and $\mathcal{A}_i$ are obtained in the same manner as in Algorithm~\ref{algo:train}. (2).

        \textit{\textbf{(2) Unsupervised Multi-graph Matching.}}

        \STATE Universe matching of $\mathcal{G}_i$: $U_i = Sinkhorn(\mathcal{V}_i^t \ \mathcal{U}^{\mathsf{T}}, \tau)$.

        \FOR{ $(\mathcal{G}_i, \mathcal{G}_j)$ in $\{ \mathcal{G}_1, \cdots, \mathcal{G}_m \}$}
        \STATE Affinity matrix $M_{ij} \leftarrow f_{mlp}\{\mathcal{V}_i^t \mathcal{W}_{x}^t\cdot(\mathcal{V}_j^t \mathcal{W}_{y}^t)^{\mathsf{T}}\}$, where $\mathcal{W}_{x}^t$ and $\mathcal{W}_{y}^t$ are two learnable linear projection, $f_{mlp}$ is a multi-layer perception (MLP).

        \REPEAT

        \STATE $V_i \leftarrow V_i + (\lambda \mathcal{A}_i^t U_i U_j^{\mathsf{T}} \mathcal{A}_j^t U_j + M_{ij} U_j)$.

        \STATE Relax $V_i$ to the space of universe: $U_i \leftarrow sinkhorn(V_i, \tau)$. 

        \UNTIL{$\{U_i\}$ \textit{converged} OR \textit{Iter} $>$ \textit{MIter}}

        \ENDFOR
        
    \STATE Fine-tune the segmentation network (update $E$ and $S$) using $L_{mat}$ in Eq. (12).

    \textit{\textbf{(3) Final Inference Process.}}
    
    \STATE  $y^{*}_i = S(E(x_i))$

    \end{algorithmic}
    \label{algo:test}
\end{algorithm*}


    

\end{document}